\documentclass[letterpaper,12pt,oneside,final]{article}

\newcommand{\href}[1]{#1} 

\usepackage{ifthen}
\newboolean{ElectronicVersion}
\setboolean{ElectronicVersion}{true} 

\usepackage{amsmath,amssymb,amstext} 
\usepackage[pdftex]{graphicx} 
\usepackage{subfig}
\usepackage{booktabs}
\usepackage{epstopdf}
\usepackage{color}

\usepackage[top=1in, bottom=1in, left=1in, right=1in]{geometry} 

\usepackage{algcompatible}
 \usepackage{algorithm}
 
 \usepackage{cite}

\usepackage{float}

\ifthenelse{\boolean{ElectronicVersion}}{
    \usepackage[letterpaper=true,pdftex,bookmarks,pagebackref,
	plainpages=false, 
        pdfpagelabels=true, 
        pdfauthor=Brendan Ames 
        ]{hyperref}}{}

\let\origdoublepage\cleardoublepage
\newcommand{\clearemptydoublepage}{%
  \clearpage{\pagestyle{empty}\origdoublepage}}
\let\cleardoublepage\clearemptydoublepage

\newtheorem{theorem}{Theorem}[section]

\newtheorem{corollary}{Corollary}[section]

\newtheorem{lemma}{Lemma}[section]

\newcommand{\A}{\bs{A}}

\renewcommand{\S}{\bs{S}}
\renewcommand{\O}{\mathcal{O}}

\newcommand{\M}{\bs{M}}
\newcommand{\R}{\mathbf{R}}

\newcommand{\C}{\bs{C}}
\newcommand{\N}{\bs{N}}

\newcommand{\B}{\bs{B}}
\newcommand{\W}{\bs{W}}
\newcommand{\D}{\bs{D}}

\DeclareMathOperator*{\argmin}{\arg\min}
\DeclareMathOperator*{\argmax}{\arg\max}

\DeclareMathOperator{\diag}{{diag}}
\DeclareMathOperator{\Diag}{{Diag}}

\DeclareMathOperator{\Null}{{Null}}

\DeclareMathOperator{\sign}{{sign}}

\DeclareMathOperator{\st}{{s.t.}}

\newcommand{\bs}{\boldsymbol}
\renewcommand{\a}{\bs{a}}
\renewcommand{\b}{\bs{b}}

\newcommand{\s}{\bs{s}}

\newcommand{\bv}{\bs{v}}
\newcommand{\x}{\bs{x}}
\newcommand{\X}{\bs{X}}

\newcommand{\y}{\bs{y}}

\newcommand{\z}{\bs{z}}
\newcommand{\w}{\bs{w}}

\newcommand{\abs}[1]{\ensuremath{\left| #1  \right| }}
\newcommand{\branchdef}[1] {\ensuremath{ \left\{\begin{array}{rl} #1 \end{array} \right. }} 
\newcommand{\bbra}[1]{\ensuremath{ \big( #1 \big) } } 
\newcommand{\Bbra}[1]{\ensuremath{ \Big( #1 \Big) } } 
\newcommand{\rbra}[1]{\ensuremath{\left( #1 \right)}} 
\newcommand{\bra}[1]{\ensuremath{\left\{ #1 \right\}}} 

\newcommand{\qu}[1]{\lq\lq{#1}\rq\rq}

\newcommand{\seq}[2]{\left\{ #1 \right\}^{\infty}_{#2 = 0} }

\newcommand{\ra}{\rightarrow}

\newcommand{\qed}{\hfill\rule{2.1mm}{2.1mm}}

\newcommand{\half}{\frac{1}{2}}
\newcommand{\shalf}{\textstyle{\frac{1}{2}}}

\numberwithin{equation}{section}
\numberwithin{figure}{section}
\numberwithin{table}{section}

\title{Alternating direction method of multipliers for penalized zero-variance discriminant analysis}
\author{Brendan P.W.~Ames
\thanks{Department of Mathematics,
The University of Alabama,
Box 870350,
Tuscaloosa, AL 35487-0350,  \href{mailto:bpames@ua.edu}{bpames@ua.edu}
}
 \and
Mingyi Hong
\thanks{  Department of Industrial and Manufacturing Systems Engineering,
Iowa State University,
Black Engineering Building,
Ames, IA, 50011,
  \href{mailto: mingyi@iastate.edu}{mingyi@iastate.edu} }
  }

\begin{document}
\allowdisplaybreaks


\maketitle
\begin{abstract}
We consider the task of classification in the high dimensional setting where
the number of features of the given data is significantly greater than the number of observations.
To accomplish this task, we propose a heuristic, called sparse zero-variance discriminant analysis (SZVD), for simultaneously performing linear discriminant analysis and feature selection on high dimensional data. 
This method combines classical zero-variance discriminant analysis, where
discriminant vectors are identified in the null space of the sample within-class covariance matrix,
with penalization applied to  induce sparse structures in the resulting vectors.
To approximately solve the resulting nonconvex problem, we develop a simple algorithm based on the alternating direction method of multipliers.  Further, we show that this algorithm is applicable to a larger class of penalized generalized eigenvalue problems,
including a particular relaxation of the sparse principal component analysis problem.
Finally, we establish theoretical guarantees for convergence
of our algorithm to stationary points of the original nonconvex problem, and empirically demonstrate the effectiveness of our heuristic
for classifying simulated data and
data drawn from applications in time-series classification.

\end{abstract}
\allowdisplaybreaks

\section{Introduction}
A standard technique for supervised classification is to perform dimensionality reduction to project
the data to a lower dimensional space where the classes are well separated, and then classify the data in this lower dimensional space.
For example, a classical approach of
Fisher~\cite{fisher1936use}, \cite{welling2005fisher}, \cite[Chapter 4]{elements}, called \emph{linear discriminant analysis (LDA)},
is to define a linear mapping to a lower
dimensional space where the ratio of the sample between-class scatter and within-class scatter of the projected training data is maximized,
and then assign each test observation to the class of the nearest projected centroid.
When the training data is randomly sampled from one of two Gaussian distributions with common covariance matrix,
this approach reduces to Bayes' rule.

Before proceeding further, we first define some notation and assumptions.
We represent our  given data  as the matrix $\X \in \R^{n\times p};$
here, each row $\x_i \in \R^p$  of $\X$ represents  a single observation consisting of  $p$ features and the data set contains $n$ such observations.
We assume that the data has been centered and normalized so that each feature has mean equal to $0$ and variance
equal to $1$.
Further, we assume each observation is labeled as belonging to exactly one of $k$ classes, denoted $C_1, C_2, \dots, C_k$.
Considered as separate data sets, the mean and covariance of each class $C_i$  may be approximated by
the sample class-mean $\bs{\hat\mu}_i$ and covariance matrix $\bs{\hat\Sigma}_i$ given by
$$
	\bs{\hat \mu}_i = \sum_{j \in C_i} \frac{\x_j}{|C_i|} \hspace{0.5in} 
	\bs{\hat \Sigma_i} = \frac{1}{n} \sum_{j\in C_i}(\x_j - \bs{\hat \mu}_i)(\x_j - \bs{\hat \mu}_i)^T.
$$
We estimate variability within classes using the \emph{sample within-class covariance matrix} $\W$, defined as
 the sum of the sample class-covariance matrices
 $$
	\W = \sum_{i=1}^k \bs{\hat\Sigma}_i
	=  \frac{1}{n} \sum_{i=1}^k \sum_{j\in C_i}(\x_j - \bs{\hat \mu}_i)(\x_j - \bs{\hat \mu}_i)^T,
$$
and estimate variability between classes using the  \emph{sample between-class covariance matrix},
\begin{equation*} \label{eq: B def}
	\B = \frac{1}{n} \sum_{i=1}^k |C_i|\bs{\hat \mu}_i \bs{\hat \mu}_i^T;
\end{equation*}
note that $\B$ is defined as a rescaling of the sample covariance matrix of the data set obtained by replacing
each observation with the sample mean of its class.
As such, the column space of $\B$ is spanned by the $k$ linearly dependent vectors $\bs{\hat\mu}_1, \bs{\hat\mu}_2, \dots,
\bs{\hat\mu}_k$ and, thus, has rank at most $k-1$.

As mentioned earlier, we would like to identify a projection of the rows of $X$ to a lower
dimensional space where the projected class means are well separated,
while observations within the same class are relatively
close in the projected space. To do so, LDA seeks a set of nontrivial loading vectors
$\w_1, \w_2, \dots, \w_{k-1}$, obtained
 by repeatedly maximizing the criterion
\begin{equation} \label{eq: LDA obj}
	J(\w) = \frac{\w^T \B \w}{\w^T \W\w}.
\end{equation}	
The fact that $\B$ has rank at most $k-1$ suggests that there exist at most $k-1$ orthogonal directions $\w_1, \w_2, \dots, \w_{k-1}$ such that the quadratic form
$ \w_i^T \B \w_i$ has nonzero value.
%

To perform dimensionality reduction using LDA, we identify the desired loading vectors  $\w_1, \w_2, \dots, \w_{k-1}$  by sequentially solving the optimization problem
\begin{equation} \label{eq: LDA frac}
	\w_{i} = \argmax_{\w\in\R^p} \bra{ \frac{\w^T \B \w}{\w^T \W \w}: \w^T \W \w_j = 0 \; \forall \, j =1,\dots, i-1 }
\end{equation}	
for all $i=1,2,3,\dots, k-1$.
That is, $\w_i$ is the vector $\W$-conjugate to the span of
 $\{\w_1,$ $\w_2,$$\dots,$$\w_{i-1}\}$ that maximizes the LDA criterion \eqref{eq: LDA obj}.
Noting that the criterion $J(\w)$ is invariant to scaling, we may assume that $\w^T \W \w \le 1$ and rewrite \eqref{eq: LDA frac} as
\begin{equation} \label{eq: LDA}
	\w_{i} = \argmax_{\w\in\R^p} \bra{ \w^T \B \w: \w^T \W \w \le 1, \; \w^T \W \w_j = 0 \; \forall \, j =1,\dots, i-1  }.
\end{equation}	
Thus, finding the $k-1$ discriminant vectors  $\{\w_1,\w_2,\dots,\w_{i-1}\}$ is equivalent to solving the generalized eigenproblem \eqref{eq: LDA}.
When the sample within-class covariance matrix $\W$ is nonsingular, we may solve \eqref{eq: LDA} by performing the change of variables
$\z = \W^{1/2} \w$.
After this change of variables, we have $\w_i = \W^{-1/2} \z_i$, where 
\begin{equation*} 
	\z_i = \argmax_{\z \in \R^p} \bra{ { \z^T \W^{-1/2} \B \W^{-1/2} \z} : \z^T \z \le 1, \,\z^T \z_j = 0, \, j=1,\dots, i-1}.
\end{equation*}	
That is, we may find the desired set of discriminant vectors by finding the set of nontrivial unit eigenvectors of $\W^{-1/2} \B \W^{-1/2}$ and multiplying each eigenvector by $\W^{1/2}$.

In the high dimension, low sample size setting, $p > n$,
the sample within-class covariance matrix $\W$ is singular.
Indeed, $\W$ has rank at most $n$, as it is a linear combination of  $n$ rank-one matrices.
In this case, the change of variables $\z = \W^{1/2}\w$ used to compute the discriminant vectors is not well-defined.
Moreover, the objectives of \eqref{eq: LDA frac} and \eqref{eq: LDA} can be made arbitrarily large if there
exists some vector $\w$ in the null space of $\W$ not belonging to the null space of $\B$.
To address this singularity issue, and to increase interpretability of the discriminant vectors, we
propose a new heuristic for performing linear discriminant analysis in this high dimension, low sample size setting, based on projection onto the null space
of $\W$, while simultaneously performing feature selection using sparsity inducing penalties.

The primary contributions of this paper are twofold.
First, we propose a new heuristic for penalized classification based on $\ell_1$-penalized zero-variance discriminant analysis;
we provide a brief overview of zero-variance discriminant analysis in Section~\ref{sec: ZFD}.
The use of an $\ell_1$-norm regularization term (or other sparsity inducing penalty) encourages sparse
loading vectors for computing the low-dimensional representation, allowing further improvement in computational
efficiency and interpretability.
This approach in itself is not novel; $\ell_1$-regularization and similar techniques have long been used in the statistics, machine learning,
and signal processing communities to induce sparse solutions, most notably in the LASSO \cite{tibshirani1996regression} and compressed sensing
\cite{candes2006compressive,kutyniok2012compressed} regimes.
However, our combination of $\ell_1$-regularization and zero-variance discriminant analysis is novel and, based on experimental
results summarized in Section~\ref{sec: expts}, appears to improve on the current state of the art in terms of both computational
efficiency and quality of discriminant vectors;
a brief review of $\ell_1$-regularization for high dimensional LDA may be found in Section~\ref{sec: PLDA}.

Second, we propose a new algorithm for obtaining approximate solutions of penalized eigenproblems of the form \eqref{eq: generic prob}, based
on the alternating direction method of multipliers.
Our algorithm finds an approximate solution of \eqref{eq: generic prob}
by alternatingly maximizing each term of the objective function until convergence.
Although the problem \eqref{eq: generic prob} has nonconcave objective in general, we will see that it is easy to maximize
each term of the objective, either $\x^T \B \x$ or $-\|\D\x\|_1$, with the other fixed.
We develop this algorithm in Section~\ref{sec: PZFD}, as motivated by its use as a heuristic for penalized zero-variance discriminant
analysis.
Further, we show that this algorithm converges to a stationary point of \eqref{eq: generic prob}
under certain assumptions on the matrices $\W$ and  $\D$ in Section~\ref{sec: convergence analysis},
and empirically test performance of the resulting classification heuristic in Section~\ref{sec: expts}.

\section{Linear Discriminant Analysis in the High Dimension, Low Sample Size Setting}
\label{LDA}

\renewcommand{\bs}{\boldsymbol}
\subsection{Zero-Variance Linear Discriminant Analysis}
\label{sec: ZFD}

As discussed in the previous section, 
the sample within-class covariance matrix $\W$ is singular in the high dimension, low sample size setting
that occurs when $p > n$.
Several solutions for this singularity problem have been proposed in the literature.
One such proposed solution is to replace $\W$ in \eqref{eq: LDA} with a positive definite approximation $\tilde \W$,
e.g., the diagonal estimate $\tilde \W = \Diag(\diag(\W))$;
see \cite{friedman1989regularized,krzanowski1995discriminant,dudoit2002comparison,bickel2004some,xu2009modified}.
Here, $\Diag(\x)$ denotes the $p\times p$ diagonal matrix with diagonal entries given by $\x \in \R^p$
and $\diag(\X)$ denotes the vector in $\R^p$ with entries equal to those on the diagonal of $\X\in\R^{p\times p}$.
After replacing the sample within-class covariance with a positive definite approximation $\tilde \W$, we obtain
a set of discriminant vectors maximizing the modified LDA criterion by sequentially solving the generalized eigenproblems
\begin{equation} \label{eq: LDA diag approx}
	\w_{i} = \argmax_{\w\in\R^p} \bra{ \w^T \B \w: \w^T \tilde \W \w \le 1, \; \w^T \tilde \W \w_j = 0 \; \forall \, j =1,\dots, i-1 }
\end{equation}	
as before.
Unfortunately,
it may often be difficult to obtain a good positive definite approximation of the population within-class scatter matrix.
For example, the diagonal approximation ignores any correlation between features, while approximations based on
perturbation of $\W$ may require some training to obtain a suitable choice of $\tilde \W$.

On the other hand,  \emph{zero-variance discriminant analysis}  (ZVD), as proposed by Krzanowski et al.~\cite{krzanowski1995discriminant},
embraces the singularity of $\W$  and seeks a set of discriminant vectors
belonging to the null space of $\W$.
If $\Null(\W) \nsubseteq \Null(\B)$,
we may obtain a nontrivial
discriminant vector by solving the generalized eigenproblem
\begin{equation} \label{eq: ZVD}
	\max_{\w\in\R^p} \bra{\w^T \B \w :  \W \w = \bs 0, \; \w^T \w \le 1 };
\end{equation}
we may find the remaining nontrivial zero-variance discriminant vectors following a deflation process as before.
That is, in ZVD we seek  orthogonal directions $\w$ belonging to $\Null(\W)$ maximizing between-class scatter;
because we are restricting our search to $\Null(\W)$, we seek orthogonal directions, not $\W$-conjugate directions as before.
If the columns of $\N \in \R^{p\times d}$ form an orthonormal basis for $\Null(\W)$, then ZVD is equivalent to the eigenproblem
\begin{equation} \label{eq: ZVD eig}
	\max_{\x\in \R^d} \bra{ \x^T (\N^T \B \N) \x : \x^T \x \le 1 },
\end{equation}	
where $d$ denotes the dimension of $\Null(\W)$.
Clearly, the dimension of $\Null(\W) \setminus \Null(\B)$ may be less than $k-1$.
In this case, $\N^T \B \N$ has less than $k-1$ nontrivial eigenvectors;
a full set of $k-1$ discriminant vectors can be obtained by
searching for the remaining discriminant vectors in the complement of $\Null(\W)$ (see \cite[pp.~8-9]{duintjer2007improving}).
Alternately, reduced rank LDA could be performed using only the nontrivial discriminant vectors found in $\Null(\W)$.

\subsection{Penalized Linear Discriminant Analysis}
\label{sec: PLDA}

While ZVD and/or the use of a positive definite approximation of the within-class covariance matrix
may solve the singularity problem in certain cases,
particularly when the classes are separable in the subspace defined by the zero-variance discriminant vectors,
the resulting discriminant vectors are typically difficult to interpret, especially when $p$ is very large.
Indeed, the method reduces to  solving  generalized eigenproblems and there is no reason to expect the obtained eigenvectors
to possess any meaningful structure.
Ideally, one would like to simultaneously perform feature selection and dimensionality reduction to obtain a set of discriminant vectors containing relatively few nonzero entries,
or some other desired structure.
In this case, one would be able to identify which features are truly meaningful for the task of the dimensionality reduction,
while significantly improving computational efficiency through the use of sparse loading vectors.

Several recent articles have proposed variants of LDA involving
$\ell_1$-regularization as a surrogate for the restriction that the
discriminant vectors be sparse
in order
to increase interpretability of the obtained loading vectors.
In \cite{witten2011penalized}, Witten and Tibshirani propose a penalized version of LDA where the $k$th discriminant vector is the solution of
the optimization problem
\begin{equation} \label{eq: WT LDA}
	\w_k = \argmax_{\w \in \R^p} \bra{ \w^T \B \w - \rho(\w) : \w^T \tilde \W \w \le 1, \; \w^T \tilde \W \w_i = 0 \; \forall\, i \le k-1 },
\end{equation}	
where $\rho:\R^p \ra \R_+$ is either an $\ell_1$-norm or fused LASSO penalty function, and  $\tilde \W$
is the diagonal estimate of the within-class covariance $\tilde \W = \Diag(\diag(\W))$.
The optimization problem \eqref{eq: WT LDA} is nonconvex, because $\B$ is positive semidefinite, and cannot be solved
as a generalized eigenproblem due to the presence of the regularization term $\rho(\w)$.
Consequently, it is unclear if it is possible to solve \eqref{eq: WT LDA} efficiently.
Witten and Tibshirani propose a minorization algorithm for approximately solving \eqref{eq: WT LDA};
the use of the diagonal estimate $\tilde \W$ is partially motivated by its facilitation of the use
of soft thresholding when solving the subproblems arising in this minorization scheme.

Clemmensen et al.~\cite{clemmensen2011sparse} consider an iterative method for penalized regression to  obtain
sparse discriminant vectors.
Specifically, Clemmensen et al.~apply an elastic net penalty \cite{zou2005regularization} to the optimal scoring formulation of the
LDA classification rule discussed in \cite{hastie1995penalized} as follows.
Suppose that the first $\ell-1$ discriminant vectors $\w_1, \w_2,\dots, \w_{\ell-1}$ and scoring vectors $\boldsymbol\theta_1, \boldsymbol\theta_2, \dots, \boldsymbol\theta_{\ell-1}$ have been
computed. Then the $\ell$th discriminant vector $\w_\ell$ and scoring  vector $\boldsymbol\theta_\ell$ are the optimal solution pair of the problem
\begin{equation} \label{eq: SDA}
	\begin{array}{rl} \displaystyle \min_{\w, \boldsymbol\theta} & \|\bs Y \boldsymbol\theta - \bs X \w \|^2 + \lambda_1 \w^T \bs \Omega \w + \lambda_2 \|\w\|_1 \\
						\st &\boldsymbol \theta^T \bs Y^T \bs Y \boldsymbol\theta = n, \; \boldsymbol\theta^T \bs Y^T \bs Y \boldsymbol\theta_\ell = 0 \, \;\forall \, \ell < k.
	\end{array}
\end{equation}
Here $\bs Y$ is the $n\times k$ partition matrix of the data set $\bs X$, i.e., $Y_{ij}$ is the binary indicator variable
for membership of the $i$th observation in the $j$th class,
$\bs \Omega$ is a positive definite matrix chosen to ensure that $\bs W +\bs  \Omega$ is positive definite and to encourage
smoothness of the obtained discriminant vectors,
and $\lambda_1$ and $\lambda_2$ are nonnegative tuning parameters controlling the ridge regression and $\ell_1$-penalties, respectively.
Clemmensen et al.~propose the following iterative alternating direction method for solving \eqref{eq: SDA}.
Suppose that  we have the approximate solutions $\tilde \w_i$ and $\tilde {\boldsymbol \theta_i}$ of \eqref{eq: SDA} at the $i$th step.
These approximations are updated by first solving \eqref{eq: SDA} for $\tilde \w_{i+1}$ with $\boldsymbol\theta$ fixed (and equal to $\tilde{\boldsymbol \theta}_i$);
$\tilde{\boldsymbol \theta}_{i+1}$ is then updated by solving \eqref{eq: SDA} with $\w$ fixed (and equal to $\tilde \w_{i+1})$.
This process is repeated until the sequence of approximate solutions has converged or a maximum number of iterations has been performed.
It is unclear if this algorithm is converging to a local minimizer because the criterion \eqref{eq: SDA} is nonconvex; however,
it can be shown that the solution of \eqref{eq: WT LDA} is a  stationary point of \eqref{eq: SDA} under mild assumptions (see \cite[Sect.~7.1]{witten2011penalized}).

In addition to these heuristics, Cai and Liu~(LPD)~\cite{cai2011direct} propose a constrained $\ell_1$-minimization for directly estimating the linear discriminant classifier for the two-class case,
and establish guarantees for both consistency and sparsity of the classifier.
Moreover, several thresholding methods \cite{tibshirani2003class,
guo2007regularized, shao2011sparse} for sparse LDA have also been proposed;
a summary and numerical comparison of several of these cited methods can be found in \cite{clem2013survey}.

\subsection{Penalized Zero-Variance Linear Discriminant Analysis}
\label{sec: PZFD}
Following the approach of \cite{witten2011penalized, clemmensen2011sparse, cai2011direct},
we propose a modification of zero-variance discriminant analysis, which we call SZVD, where
 $\ell_1$-penalization is employed to induce sparse discriminant vectors.
 Specifically, we  solve the problem
\begin{equation} \label{eq: PZFD 0}
	\max_{\w\in\R^p} \bra{  \shalf\w^T \B \w - \gamma \sum_{i=1}^p \sigma_i \abs{\rbra{ \D \w}_i} :  \W\w =\bs0, \w^T\w \le 1 }
\end{equation}
to obtain the first discriminant vector; if the discriminant vectors $\w_1, \dots, \w_{k-1}$ have been identified,
$\w_k$ can be found by appending $\{\w_1^T, \dots, \w_{k-1}^T\}$ to the rows of $\W$ and solving \eqref{eq: PZFD 0} again.
Here, $\D \in \O^p$ is an orthogonal  matrix, and the $\ell_1$-penalty acts as a convex surrogate
for the cardinality of $\w$ with respect to the basis given by the columns of $\D$.
The parameter $\boldsymbol\sigma \in \R^p$ is a scaling vector used to control emphasis of penalization;
for example, the scaling vector $\boldsymbol\sigma$ may be taken to be the within-class standard deviations of the features $\boldsymbol\sigma = \sqrt{\diag \,\W}$
to ensure that a greater penalty is imposed on features that vary the most within each class.
As before, letting the columns of $\N \in \R^{p\times d}$ form an orthonormal basis for $\Null(\W)$ yields the equivalent formulation
\begin{equation} \label{eq: PZFD}
	\max_{\x\in\R^d} \bra{  \shalf \x^T \N^T\B\N \x - \gamma \sum_{i=1}^p \sigma_i \abs{\rbra{ \D \N\x}_i} : \x^T\x \le 1 }.
\end{equation}
Like \eqref{eq: WT LDA}, \eqref{eq: PZFD} is the maximization of a difference of convex functions over the unit ball;
it is unknown if an efficient algorithm for solving \eqref{eq: PZFD} exists, although maximizing nonconcave functions is NP-hard in general.
In the following section, we develop an algorithm to find stationary points of problem \eqref{eq: PZFD}.
Our algorithm is based on the  \emph{alternating direction method of multipliers} and we will use it as a heuristic
for classification in Section~\ref{sec: expts}.

\section{SZVD: An Alternating Direction Method of Multipliers Approach for Penalized Zero-Variance Linear Discriminant Analysis}
\label{sec: PZFD}

In this section, we develop an algorithm, called \emph{SZVD}, for finding stationary points of the problem \eqref{eq: PZFD}  based on the alternating direction
method of multipliers;
we will use these stationary points as estimates of the global solution of \eqref{eq: PZFD} for the
purpose of obtaining sparse discriminant vectors of a given data set.
We develop our algorithm in Section~\ref{sec: ADMM} and establish convergence in Section~\ref{sec: convergence analysis}.

\subsection{The Alternating Direction Method of Multipliers}
\label{sec: ADMM}

Given problems of the form
$
	\min_{\x, \y}  \bra{ f(\x) + g(\y): \bs A\x + \B \y = \mathbf{c}},
$
the \emph{alternating direction method of multipliers}  (ADMM) generates a sequence of iterates by alternately minimizing the augmented Lagrangian of the problem
with respect to each primal decision variable, and then updating the dual variable using an approximate dual ascent;
a recent survey on the ADMM and related methods can be found in \cite{boyd2011distributed}.

To transform \eqref{eq: PZFD} to a form appropriate for the ADMM, we define an additional decision variable $\y \in \R^p$ by
$\y = \D\N \x $.
After this splitting of variables and replacing the maximization with an appropriate minimization, \eqref{eq: PZFD} is equivalent to
\begin{equation} \label{eq: PZFD xy 0}
	\min_{\x \in \R^d, \y\in \R^p}
		\bra{ -\shalf \x^T (\N^T \B \N) \x  + \gamma \sum_{i=1}^p \sigma_i \abs{y_i} : \y^T \y \le 1, \; \D\N \x = \y}.
\end{equation}		
Letting $\A = \N^T \B \N$ and $\rho(\y) = \sum_{i=1}^p \sigma_i \abs{y_i}$, we see that \eqref{eq: PZFD xy 0} is equivalent to
\begin{equation} \label{eq: PZFD xy}
	\min_{\x \in \R^d, \y\in \R^p}
		\bra{ -\shalf \x^T \A\x  + \gamma \rho(\y) : \y^T \y \le 1, \; \D\N \x - \y = \bs 0}.
\end{equation}		
This transformation has the additional benefit that $\rho(\y)$ is separable in $\y$, while $\rho(\D\N\x)$ is not separable in $\x$;
this fact will play a significant role in our ADMM algorithm, as we will see shortly.

The problem \eqref{eq: PZFD xy} has augmented Lagrangian
$$
	L_\beta(\x,\y,\z) = -\half \x^T \A \x + \gamma \rho(\y) + \delta_{B_p}(\y) + \z^T(\D\N\x - \y) + \frac{\beta}{2} \|\D\N\x - \y\|^2,
$$
where $\beta > 0$ is a regularization parameter chosen so that $L_\beta$ strictly convex in each of $\x$ and $\y$, and
$\delta_{B_p}: \R^p\ra \{0, +\infty\}$ is the indicator function of the $\ell_2$-ball in $\R^p$, defined by  $\delta_{B_p}(\y) = 0$ if $\y^T\y \le 1$ and
 equal to $+\infty$ otherwise.
If we have the iterates
$(\x^t, \y^t, \z^t)$ after $t$ iterations, our algorithm obtains $(\x^{t+1}, \y^{t+1}, \z^{t+1})$
by sequentially minimizing $L_\beta$ with respect to $\y$ and $\x$ and then
updating $\z$ by an approximate dual ascent step:
\begin{align}
	\y^{t+1} &= \argmin_{\y\in \R^p} L_\beta \rbra{\x^t, \y, \z^t } \label{eq: y step} \\
	\x^{t+1} &= \argmin_{\x\in \R^p} L_\beta \rbra{\x, \y^{t+1}, \z^t } \label{eq: x step} \\
	\z^{t+1} &= \z^t + \beta \rbra{ \D\N \x^{t+1} - \y^{t+1}}. \label{eq: z step}
\end{align}
We next describe the solution of the subproblems \eqref{eq: y step} and \eqref{eq: x step}.

We begin with  \eqref{eq: y step}. It is easy to see that \eqref{eq: y step}
is equivalent to
\begin{equation} \label{eq: y prob}
	\y^{t+1} = \argmin_{\y\in \R^p} \bra{\gamma \rho(\y) +\frac{\beta}{2} \y^T \y- \y^T (\beta \D\N \x^t + \z^t): \y^T \y \le 1 }.
\end{equation}	
Applying  the Karush-Kuhn-Tucker conditions \cite[Section 5.5.3]{boydvdb} to \eqref{eq: y prob},  we see that $\y^{t+1}$ must satisfy
\begin{equation} \label{eq: y KKT}
	\bs 0 \in \gamma \partial \rho(\y^{t+1}) + (\beta + \lambda) \y^{t+1}
	- (\beta \D\N \x^t + \z^t), \hspace{0.25in}  \lambda( (\y^{t+1})^T \y^{t+1} - 1) = 0
\end{equation}
for some Lagrange multiplier $\lambda \ge 0$.
The stationarity condition and the form of the subdifferential of $\rho$ (see \cite[Section 3.4]{boyd2008subgradients})
imply that there exists some $\boldsymbol\phi \in \R^p$, satisfying $\boldsymbol\phi^T \y^{t+1} = \rho(\y^{t+1})$ and $|\phi_i| \le \sigma_i$ for all $i \in \{ 1,\dots, p\}$, such that
$$
	0 = (\beta + \lambda) y_i^{t+1} + \gamma \phi_i - b_i,
$$
where $\b = \beta \D\N \x^t + \z^t$, for each $i \in \{1,\dots, p\}.$
Rearranging and solving for $y^{t+1}_i$ shows that
$$
	(\beta + \lambda) y^{t+1}_i = \sign(b_i) \cdot \max\{|b_i| - \gamma \sigma_i, 0 \}
$$
for all $i \in \{1,\dots, p\}$.
Letting $\s^{t+1} \in \R^p$ be the vector such that $s^{t+1}_i = \sign(b_i) \cdot \max\{|b_i| - \gamma \sigma_i, 0 \}$ for all $i\in\{1,\dots, p\}$ and
applying the complementary slackness condition $ \lambda( (\y^{t+1})^T \y^{t+1} - 1) = 0$ shows that
\begin{equation} \label{eq:  y update}
	\y^{t+1} = \frac{\s^{t+1} }{\beta + \max\{\bs0, \|\s^{t+1}\| - \beta\} }.
\end{equation}
That is, we update $\y^{t+1}$ by applying soft thresholding to $\b$ with respect to $\gamma \bs\sigma$, and then
normalizing the obtained solution if it has Euclidean norm greater than 1.

On the other hand, \eqref{eq: x step} is equivalent to
\begin{equation} \label{eq: x prob}
	\x^{t+1} = \argmin_{\x\in\R^d} \half \x^T (\beta \bs I -  \A) \x  + \x^T (\D\N)^T (\z^t - \beta \y^{t+1}).
\end{equation}
For sufficiently large choice of $\beta$, \eqref{eq: x prob} is an unconstrained convex program.
Finding the critical points of the objective of \eqref{eq: x prob} shows that
$\x^{t+1}$ is the solution of the linear system
\begin{equation} \label{eq: x system}
	(\beta \bs I - \bs A) \x^{t+1} = (\D\N)^T (\beta \y^{t+1} - \z^t).
\end{equation}
The algorithm in Step~\ref{alg: ADMM}
 is stopped after $t$ iterations if the primal and dual residuals
satisfy
\begin{align*}
	\| \D\N \x^t - \y^t\| &\le \epsilon_{abs} \cdot \sqrt{p} + \epsilon_{rel} \cdot \max \{ \|\x^t \|, \|\y^t\|\},  \hspace{0.08in}
	\beta \|\y^t - \y^{t-1} \| \le  \epsilon_{abs} \cdot \sqrt{p} + \epsilon_{rel} \cdot  \|\y^t\|
\end{align*}
for desired absolute and relative error tolerances $\epsilon_{abs}$ and $\epsilon_{rel}$;
see \cite[Sect.~3.3.1]{boyd2011distributed} for motivation for this choice of stopping criteria.

\begin{algorithm}[t]
\caption{ADMM for Sparse zero-variance discriminant analysis (SZVD)}
\label{alg:SZVD}
	\begin{algorithmic}[1]
	\STATE
		Given data $\X$. Compute sample between-class covariance  $\B$
		and 	within-class covariance $\W$.
	\STATE
		Choose $k-1$ regularization parameters $\{\gamma_1, \gamma_2, \dots, \gamma_{k-1}\}$,
		and sets of initial solutions $\{(\x^0, \y^0, \z^0)_i\}_{i=1}^{k-1}$. Set $i =1$.			
		
	\REPEAT
			\STATE \label{alg: N}
			Compute basis $\N$ for the null space of $\W$.
		\STATE \label{alg: ADMM}
			Approximately solve \eqref{eq: PZFD xy}
			 with regularization parameter $\gamma = \gamma_i$ using limit of
			iterates $\seq{(\x^t, \y^t,\z^t)}{t}$ generated by
			 \eqref{eq: y update}, \eqref{eq: x system}, and \eqref{eq: z step}
			and the initial solution $(\x^0, \y^0, \z^0)_i$ to obtain $(\x^*, \y^*,\z^*)$.
			Let $\w_i = \D\N \x^* = \y^*$.
		\STATE
			Append $\w_i$ to $\W$: $[\W; \w_i^T] \mapsto \W$.
		\STATE
			Update $i=i+1$.
	 \UNTIL{(all nontrivial zero-variance  discriminant vectors $\{\w_i\}_{i=1}^{k-1}$ are computed.)}
	\end{algorithmic}
\end{algorithm}

Our heuristic for identifying the (at most) $k-1$ penalized zero-variance discriminant vectors $\w_1, \w_2,\dots, \allowbreak  \w_{k-1}$ for a given data set $\X$
using the stationary points of \eqref{eq: PZFD}
is summarized in Algorithm~\ref{alg:SZVD}.
The computation of $\N$ in Step~\ref{alg: N} uses a QR decomposition of the within-class scatter matrix $\W$, computation
of which requires $\O(p^3)$ operations.
The ADMM Step~\ref{alg: ADMM} requires a preprocessing step consisting of  Cholesky factorization of the coefficient matrix in \eqref{eq: x system} (at a cost of $\O(d^3)$ operations)
and iterations containing a constant number of operations costing no more than matrix-matrix-vector multiplications of the form $\D\N\bs v$ (each requiring $\O(p^2)$ operations).
It should be noted that the cost of this update can be significantly improved by exploiting the structure of $\A$ in some special cases.
For example, if $k=2$, we have the decomposition $\A = \bv\bv^T$, where $\bv = \N^T(\bs{\hat \mu_1} - \bs{\hat \mu_2})$.
Applying the Sherman-Morrison-Woodbury formula \cite[Equation (2.1.4)]{GV} to \eqref{eq: x system} shows that
$$
	\x^{t+1} = \frac{1}{\beta} \rbra{ (\D\N)^T(\beta \y^{t+1} - \z^t) - \rbra{\frac{(\beta \y^{t+1} - \z^t)^T \D\N \bv}{\beta - \bv^T\bv} } \bv },
$$	
which can be computed using $O(p^2)$  floating point operations (dominated by
the matrix-vector multiplication $(\D\N)^T (\beta \y^{t+1} - \z^t)$).
That is, if $k=2$, we use the available factorization $\A = \bv\bv^T$ to avoid the expensive
Cholesky factorization step in Step~\ref{alg: ADMM}.
In either case, the algorithm has a total time complexity of $O((k-1)p^3) + O(\#\mbox{its} \cdot  p^2)$,
where $\#\mbox{its}$ denotes the total number of iterations required by Step~\ref{alg: ADMM}.

\subsection{Convergence Analysis}
\label{sec: convergence analysis}
It is known that the ADMM converges to the optimal solution of
$$
	\min_{\x \in \R^{n_1}, \y\in\R^{n_2}}  \bra{ f(\x) + g(\y): \A\x + \B \y = \bs{c}}
$$
if both $f$ and $g$ are convex functions (see \cite[Theorem~8]{eckstein1992douglas},  \cite[Section~3.2]{boyd2011distributed}, and \cite{luo2012linear}) under mild assumptions on $f$, $g$, $\A$, and $\B$.
However, general convergence results for minimizing nonconvex separable functions, such as  the objective of \eqref{eq: PZFD xy}, are unknown. Recently the authors of \cite{hong14admm} have shown that the ADMM converges for certain nonconvex consensus and sharing problems. We note that the problems considered here are not covered by those considered in \cite{hong14admm}; in particular, the constraints $\D\N \x =\y$ cannot be dealt with the analysis provided in \cite{hong14admm}.

In this section, we establish that, under certain assumptions on the within-class covariance matrix $\W$ and the dictionary matrix $\D$, the ADMM algorithm described by \eqref{eq: y step}, \eqref{eq: x step}, and \eqref{eq: z step} converges to a stationary point of \eqref{eq: PZFD xy}. Let us define a new matrix $\M=\D\N$. Clearly the columns of $\M$ are also orthogonal, as we have $\M^T \M=\N^T\D^T \D \N= \bs I$.
We have the following theorem.

\begin{theorem} \label{thm: convergence}
	Suppose that the columns of the matrix $[\M, \C] \in \O^p$ form an orthonormal basis for $\R^p$.
	Suppose further that  the sequence of iterates $(\x^k, \y^k, \z^k)$ generated by \eqref{eq: y step}, \eqref{eq: x step}, and \eqref{eq: z step} satisfies
	\begin{equation} \label{eq: conv cond}
	\C^T (\z^{k+1} - \z^k) = \bs 0
	\end{equation}
	for all $k$
	and the parameter parameter $\beta$ satisfies
	\begin{equation} \label{eq: beta bound}
		\beta > \|\A\| \rbra{ \frac{\lambda_0^2 + 2}{\lambda_0^2} },
	\end{equation}	
		where $\|\A\|$ denotes the largest singular value of $\A = \N^T \B \N$, and $\lambda_0$ denotes the smallest nonzero eigenvalue of the matrix $\M\M^T$.
	Then $\seq{(\x^k, \y^k, \z^k)}{k}$ converges to a stationary point of \eqref{eq: PZFD xy}.	
\end{theorem}	

Although the assumption \eqref{eq: conv cond} that the successive difference of the multipliers belong to the null space of $\M$ may seem unrealistically restrictive,
it is satisfied for two special classes of problems.
We have the following corollary.

\begin{corollary} \label{cor: convergence for standard basis}
	Suppose that $[\M,\C]$ forms the standard Euclidean basis for $\R^p$ and $\z^0 = \M \a^0$ for some vector $\a^0 \in \R^p$ with bounded norm.
	Then \eqref{eq: conv cond} is satisfied for all $k$ and the  sequence $\seq{(\x^k, \y^k, \z^k)}{k}$ generated by \eqref{eq: y step}, \eqref{eq: x step}, and \eqref{eq: z step}
	converges to a stationary point of \eqref{eq: PZFD xy}  if $\beta$ satisfies \eqref{eq: beta bound}.
\end{corollary}

\begin{proof}
	If $[\M,\C]$ forms the standard basis, then we can write $\y\in \R^p$ as $\y= \M \bs c+\C\bs{d}$ for some $\bs{c}$ and $\bs{d}$ with appropriate dimensions.  It follows that we may decompose the subdifferential of the $\ell_1$-norm at any $\y\in\R^p$ as
	\begin{equation} \label{eq: dl1 decomp}
		\partial \rho(\y) = \partial \rho \rbra{ \M\bs{c} + \C \bs{d} } = \partial \rho\rbra{\M\bs{c}} + \partial  \rho\rbra{\C\bs{d}} = \M\partial \rho\rbra{\bs{c} } + \C \partial  \rho\rbra{\bs{d}}
	\end{equation}
	by the fact that $\rho \circ \M$ and $\rho \circ \C $ are separable functions of $\y$.
	Substituting into the gradient condition of \eqref{eq: y KKT} we have
	\begin{align*}
		\bs 0 &\in \gamma \partial \rho(\M\bs{c}^1) + (\beta + \lambda) \M \bs{c}^1 -  (\beta \M \x^0 + \z^0)  \\
		\bs 0 &\in\gamma \partial \rho(\C\bs{d}^1) + (\beta + \lambda) \C \bs{d}^1.
	\end{align*}	
	Therefore, we must have $(\beta + \lambda) \C\bs{d}^1 \in - \gamma \partial \rho (\C\bs{d}^1)$, which implies that
	$\C\bs{d}^1 = \bs 0$ by the structure of the subdifferential of the $\ell_1$-norm.
	Extending this argument inductively shows that $\C\bs{d}^k = \bs 0$ for all $k$, or equivalently, $\C^T\y^k= \bs 0$ for all $k$.
	Substituting into \eqref{eq: z step} shows that
	$$
		\C^T(\z^{k+1} -  \z^k) = \beta \C^T (\M \x^{k+1} - \y^{k+1}) = - \beta \C^T(\y^{k+1}) = \bs 0
	$$
	because the columns of $\C$ are orthogonal to those of $\M$.
	This completes the proof.	 \qed
\end{proof}

Similarly, if $\N$ has both full row and column rank, as it would if $\W = \bs0$, then our ADMM algorithm converges.
In particular, the algorithm converges when applied to \eqref{eq: SPCA} to identify the leading sparse principal component.

\begin{corollary} \label{cor: convergence for full rank}
	Suppose that $\N$ forms a basis for $\R^p$.
	Then \eqref{eq: conv cond} is satisfied for all $k$ and the  sequence $(\x^k, \y^k, \z^k)$ generated by \eqref{eq: y step}, \eqref{eq: x step}, and \eqref{eq: z step}
	converges to a stationary point of \eqref{eq: PZFD xy} if $\beta$ satisfies \eqref{eq: beta bound}.
\end{corollary}

\begin{proof}
	If $\N$ forms a basis for $\R^p$, then $\M$ also forms a basis of $\R^p$, so its null space is spanned by $\C= \bs 0$. Clearly $\C^T (\z^{k+1} - \z^k) =\bs 0$ in this case.	 \qed
\end{proof}

The remainder of this section consists of a proof of Theorem~\ref{thm: convergence}.
To establish Theorem~\ref{thm: convergence}, we will show that the value of the augmented Lagrangian of \eqref{eq: PZFD xy}
decreases each iteration and the sequence of augmented Lagrangian values is bounded below. We will then exploit the fact that
sequence of augmented Lagrangian values is convergent to show that the sequence of ADMM iterates is convergent.
We conclude by establishing that a limit point of the sequence $\{(\x^k,\y^k, \z^k)\}_{k=0}^\infty$
must be a stationary point of the original problem \eqref{eq: PZFD xy}.
We begin with the following lemma, which establishes that the augmented Lagrangian is decreasing provided the
 hypothesis of Theorem~\ref{thm: convergence} is satisfied.


\begin{lemma} \label{lem: L decreases}
	Suppose that $\C^T(\z^{k+1} - \z^k) = \bs 0$ for all $k$ and $\beta > (\lambda_0^2 + 2) \|\A\|/\lambda_0^2$. Then
	\begin{align}
		L_\beta(&\x^{k+1}, \y^{k+1}, \z^{k+1}) - L_\beta(\x^k, \y^k, \z^k)  \notag \\
		&\le - \frac{\beta}{2} \|\y^{k+1} - \y^k\|^2 - \half \rbra{ \beta - \|\A\| - \frac{2 \|\A\|^2}{\beta \lambda_0} } \|\x^{k+1} - \x^k\|^2 \label{eq: decrease bound}		
	\end{align}
	and the right-hand side of \eqref{eq: decrease bound}
	is strictly negative if $\x^{k+1} \neq \x^k$ or $\y^{k+1}\neq \y^k$.
\end{lemma}

\begin{proof}
	We will obtain the necessary bound on the improvement in Lagrangian value
	given by $L_\beta(\x^{k+1}, \y^{k+1},\allowbreak  \z^{k+1}) - L_\beta(\x^k, \y^k, \z^k) $ by decomposing the difference as
	\begin{align*}
		L_\beta &(\x^{k+1}, \y^{k+1}, \z^{k+1}) - L_\beta(\x^k, \y^k, \z^k)   \\
		&= \bbra{ L_\beta(\x^{k+1}, \y^{k+1}, \z^{k+1}) - L_\beta(\x^{k+1}, \y^{k+1}, \z^k) } \\
			&+ \bbra{ L_\beta(\x^{k+1}, \y^{k+1}, \z^{k}) - L_\beta(\x^k, \y^k, \z^k) }
	\end{align*}
	and bounding each summand in parentheses separately.
	We begin with  the first summand
	$L_\beta(\x^{k+1}, \y^{k+1}, \allowbreak \z^{k+1}) - L_\beta(\x^{k+1}, \y^{k+1}, \z^k) .$	
	Recall that $\x^{k+1}$ satisfies (cf.~\eqref{eq: x system})
	$$
		(\beta  {\bs I} - \A) \x^{k+1} = \M^T\rbra{ \beta \y^{k+1} - \z^k }.
	$$
	Multiplying \eqref{eq: z step} by $\M^T$, using the fact that $\M^T \M=\bs I$,  and substituting the formula above yields
	$$
		\M^T \z^{k+1} = \M^T \z^k+\beta \x ^{k+1} -\beta \M^T \y^{k+1}=
			\beta \x^{k+1} - \M^T ( \beta \y^{k+1} - \z^k) = \A \x^{k+1}.
	$$
	This implies that
	$$
		\|\M^T(\z^{k+1} - \z^k)\| = \|\A (\x^{k+1} - \x^k) \|	 \le \|\A\|\|\x^{k+1} - \x^k\|.
	$$
	Applying the assumption \eqref{eq: conv cond}, we have
	$$
		\lambda_0 \|\z^{k+1} - \z^k \| \le   \|\M^T(\z^{k+1} - \z^k)\|  \le \|\A\| \|\x^{k+1} - \x^k\|.
	$$
	It follows immediately that
	\begin{align}
		L_\beta(&\x^{k+1}, \y^{k+1}, \z^{k+1}) - L_\beta(\x^{k+1}, \y^{k+1}, \z^k) 	\notag \\
			&=	(\z^{k+1} - \z^k)^T (\M\x^{k+1} - \y^{k+1}) \notag \\
			&  = \frac{1}{\beta} \|\z^{k+1} -\z^k\|^2  \le \frac{\|\A\|^2}{\beta \lambda_0^2} \|\x^{k+1} -\x^k\|^2. \label{eq: L bound z}
	\end{align}
	It remains to derive the necessary bound on $ L_\beta(\x^{k+1}, \y^{k+1}, \z^{k}) - L_\beta(\x^k, \y^k, \z^k) $.	
	
	To do so, note that subproblem \eqref{eq: x prob} is strongly convex with modulus $(\beta - \|\A\|)/2$.
	Let $f(\x) = -\half \x^T \A \x + \frac{\beta}{2} \|\M \x - \y^{k+1} + \z^k/\beta\|^2$.
	Then
	\begin{align}
		L_\beta(&\x^{k+1}, \y^{k+1}, \z^k) - L_\beta(\x^k, \y^{k+1}, \z^k)  = f(\x^{k+1}) - f(\x^k)  \notag \\
		&\le - \nabla f(\x^{k+1})^T(\x^k - \x^{k+1}) - \rbra{\frac{\beta-\|\A\|}{2} } \|\x^{k+1} - \x^k\|^2 \notag \\
		& \le- \rbra{\frac{\beta-\|\A\|}{2} } \|\x^{k+1} - \x^k\|^2 	\label{eq: L bound x}
	\end{align}
	by the fact that $\x^{k+1}$ is a minimizer of $f$ and, consequently, $\nabla f(\x^{k+1}) = 0$.
	Note that \eqref{eq: L bound z} and \eqref{eq: L bound x} in tandem imply that
	\begin{align*}
		L_\beta(&\x^{k+1}, \y^{k+1}, \z^{k+1}) - L_\beta (\x^k,  \y^{k+1}, \z^k) 
		\le \frac 1 2 \rbra{\|\A\| - \beta + \frac{2\|\A\|^2}{\beta \lambda_0^2} } \|\x^{k+1} - \x^k\|^2,
	\end{align*}		
	which is strictly negative if  $\x^{k+1} \neq \x^k$,
	by the assumption that $\beta >(\lambda_0^2 + 2)\|\A\|/\lambda_0^2$.

    { 
    Similarly, the subproblem \eqref{eq: y prob} is strongly convex with modulus $\beta/2$. Let 
    $$g(\y) =  \frac{\beta}{2} \y^T \y- \y^T (\beta \M \x^t + \z^t),$$ 
    then from the optimality condition of problem \eqref{eq: y prob}, there must exist $\xi^{t+1}\in\partial \gamma\rho(\y^{k+1})$ 
    \begin{align}\label{eq:y:opt}
    \left(\xi^{k+1}+ \nabla g(\y^{k+1})\right)^T(\y - \y^{k+1})\ge 0, \quad \forall~\y\in\mathbf{R}^p.
    \end{align}       
    Then we have
    \begin{align}
		L_\beta(&\x^{k}, \y^{k+1}, \z^k) - L_\beta(\x^k, \y^{k}, \z^k) \notag\\
        & = \gamma\rho(\y^{k+1}) - \gamma\rho(\y^k) +  g(\y^{k+1}) - g(\y^k)  \notag \\
		&\le -\left(\xi^{k+1}+g(\x^{k+1})\right)^T(\y^k - \y^{k+1}) - \frac{\beta}{2}\|\y^{k+1} - \y^k\|^2 \notag \\
		& \le- \frac{\beta}{2}\|\y^{k+1} - \y^k\|^2. 	\label{eq: L bound y}
	\end{align}
	Combining \eqref{eq: L bound z}, \eqref{eq: L bound x}, and \eqref{eq: L bound y} gives the desired bound on the decrease of $L$.}
	\qed
\end{proof}

Having established sufficient decrease of the augmented Lagrangian during each iteration, we next establish that the sequence of augmented Lagrangian values is bounded and, thus, convergent. We have the following lemma.

\begin{lemma} \label{lem: bounded}
	Suppose that $\C^T(\z^{k+1} - \z^k) = \bs 0$ for all $k$ and $\beta > (\lambda_0^2 + 2)\|\A\|/\lambda_0^2.$
	Then the sequence $\{L_\beta(\x^k, \y^k,\allowbreak  \z^k)\}$ of augmented Lagrangian values  is bounded.
	As a bounded monotonic sequence,  $\{L_\beta(\x^k, \y^k, \z^k)\}$ is convergent.
\end{lemma}

\renewcommand{\seq}[2]{\{ #1\}}

\begin{proof}
	Note that the fact that $\C^T (\z^{k+1} - \z^k)=\bs 0$ implies that  $\C^T\y^{k+1} = \bs 0$ for all $k$.
	Indeed, in this case
	$$
		\bs 0 = \C^T(\z^{k+1} - \z^k) = \beta \C^T (\M\x^{k+1} - \y^{k+1}) = \beta \C^T \y^{k+1}
	$$
	because the columns of $\M$ are orthogonal to those of $\C$.
	Thus, there exists $\{\b^k\}_{k=0}^\infty \in \R^d$, with $\|\b^k\| \le 1$, such that $\y^{k} = \M\b^{k}$ for all $k$.
	In this case,
	\begin{align}
		L_\beta (\x^k, \y^k, \z^k)	
		&= -\half (\x^k)^T \A \x^k + \gamma \rho(\y^k) + (\z^k)^T ( \M \x^k - \y^k) + \frac{\beta}{2} \|\M \x^k - \y^k\|^2  \notag\\
		&= -\half (\x^k)^T \A \x^k + \gamma \rho(\y^k) + (\z^k)^T \M(\x^k - \b^k) + \frac{\beta}{2}  \|\M \x^k - \y^k\|^2   \notag\\
		&= -\half (\x^k)^T \A \x^k + \gamma \rho(\y^k) + (\A\x^k)^T (\x^k - \b^k) + \frac{\beta}{2}  \|\M \x^k - \y^k\|^2  \label{eq: use Ax = DN z}\\
		&= -\half (\x^k)^T \A \x^k + \gamma \rho(\y^k) + (\A\x^k)^T (\x^k - \b^k) + \half (\b^k)^T \A \b^k \notag \\
		& \hspace{0.5in}- \half (\b^k)^T \A \b^k +  \frac{\beta}{2} \|\M \x^k - \y^k\|^2  \label{eq: add substract bform}\\
		&= \half \Bbra{ \| \A^{1/2} (\x^k -\b^k) \|^2  - (\b^k)^T \A \b^k + \beta \|\M \x^k - \y^k \|^2 } + \gamma \rho(\y^k)\notag
	\end{align}	
	where \eqref{eq: use Ax = DN z} follows from the identity $\M^T \z^k = \A\x^k$ and \eqref{eq: add substract bform}
	follows from adding and subtracting $\half (\b^k)^T \A \b^k$.
	Note that $\A^{1/2}$ is well-defined since $\A$ is a positive semidefinite matrix.
	Because both $\seq{\b^k}{k}$ and $\seq{\y^k}{k}$ are bounded, we may conclude that the sequence $\seq{L_\beta(\x^k,\y^k,\z^k)}{k}$ is lower bounded.	
	\qed
\end{proof}	

As an immediate consequence of Lemma~\ref{lem: bounded}, we see that each of the sequences $\seq{\x^k}{k}$, $\seq{\y^k}{k}$, and $\seq{\z^k}{k}$ is convergent.
Indeed, we have the following corollary.

\begin{corollary} \label{cor: sequences converge}
	Suppose that $\C^T (\z^{k+1} - \z^k)= \bs 0$ for all $k$  and $\beta > (\lambda_0^2 + 2)\|\A\|/\lambda_0^2.$
	Then $\seq{\x^k}{k}$, $\seq{\y^k}{k}$, $\seq{\z^k}{k},$ and $\seq{\M \x^k - \y^k}{k}$ are convergent, with
	$$
		\lim_{k\ra \infty} \M \x^k - \y^k = \bs 0.
	$$	
\end{corollary}

\begin{proof}
	The fact that $L_\beta(\x^{k+1}, \y^{k+1}, \z^{k+1}) - L_\beta(\x^k, \y^k, \z^k) \ra \bs 0$ and \eqref{eq: decrease bound} imply that
	 $\x^{k+1} - \x^k \ra \bs 0$ and $\y^{k+1} - \y^k \ra \bs 0$.
	The assumption that $\C^T (\z^{k+1} - \z^k) =\bs0$ and the identity $\M^T \z^k = \A\x^k$
	implies that 	
	$$
		\M^T(\z^{k+1} - \z^k) = \A(\x^{k+1} - \x^k) \ra \bs 0.
	$$
	Thus, $\z^{k+1} - \z^k \ra \bs 0$ because the columns of $[\M, \C]$ form an orthonormal basis for $\R^p$, which further implies that the
	constraint violation satisfies $\M \x^{k} - \y^k	 = \D \N \x^{k} - \y^k	 \ra \bs 0$.
	This completes the proof\qed
\end{proof}

It remains to establish the following lemma, which states that any limit point of the sequence generated by \eqref{eq: y step}, \eqref{eq: x step}, and \eqref{eq: z step}
is a stationary point of \eqref{eq: PZFD xy}. This step is quite straightforward, but we include it here for completeness.

\begin{lemma} \label{lem: converge to stationary points}
	Let $\bar \x, \bar \y, \bar \z$ be limit points of the sequences $\seq{\x^k}{k}$, $\seq{\y^k}{k}$, and $\seq{\z^k}{k},$ respectively.
	Then
	\begin{align}
		\bar \x &= \argmin_{\x \in \R^d} \bra{ \x^T \A \x + \bar \z^T (\M \x - \bar \y)  } \label{eq: x stationary}\\
		\bar \y &= \argmin_{\y\in \R^p} \bra{ \gamma \rho(\y) + \bar \z^T (\M \bar \x - \y) : \y^T \y \le 1 }  \label{eq: y stationary}\\
		\bar \y &= \M\bar \x . \label{eq: x = y}
	\end{align}
	Therefore, $(\bar\x, \bar\y, \bar\z)$  is a stationary point of $L_0(\bar\x,\bar\y,\bar\z)$, i.e., $\bs 0\in \partial L_0(\bar\x,\bar\y,\bar\z)$.
\end{lemma}

\begin{proof}
	That \eqref{eq: x = y} holds is a consequence of the fact that $\M \x^{k} - \y^k	 \ra \bs 0$.
	The fact that $\bs 0\in \partial L_0(\bar\x,\bar\y,\bar\z)$ if \eqref{eq: x stationary} and \eqref{eq: y stationary} hold
	is an immediate consequence of the optimality conditions for the subproblems for $\x$ and $\y$ applied at $(\bar\x,\bar\y, \bar\z)$.
	It remains to prove that \eqref{eq: x stationary} and \eqref{eq: y stationary} hold.
	
	We begin with \eqref{eq: x stationary}; \eqref{eq: y stationary} will follow by a similar argument.
	Fix $k$. Recall that $\x^{k+1}$ is a minimizer of the function
	$f(x) = -\half \x^T \A \x + \frac{\beta}{2} \|\M \x - \y^{k+1} + \z^k/\beta\|^2$.
	Therefore, $\x^{k+1}$ satisfies $\nabla f(\x^{k+1}) = 0$
	for all $\x \in \R^d$. Evaluating the gradient of $f$ at $\x^{k+1}$ shows that
	\begin{align*}
		0 &\le (\x-\x^{k+1})^T (-\A \x^{k+1} + \beta( \x^{k+1} - \M^T \y^{k+1}) + \M^T \z^k) \\
			&= (\x-\x^{k+1})^T (-\A \x^{k+1} + \M^T \z^{k+1})
	\end{align*}
	by \eqref{eq: z step}.
	This implies that $\x^{k+1}$ is also a minimizer of $-\half \x^T \A \x + \x^T \M^T \z^{k+1}$.
	It follows that
	$$
		-\half \x^T \A \x + (\z^{k+1})^T(\M \x - \y^{k+1}) \ge - \half (\x^{k+1})^T \A \x^{k+1} +  (\z^{k+1})^T(\M \x^{k+1} - \y^{k+1})
	$$
	for all $\x \in \R^d$.	
	Taking the limit as $k \ra \infty$ shows that
	$$
		-\half \x^T\ \A \x + \bar\z^T(\M \x -\bar \y) \ge - \half \bar\x^T \A \bar\x\ +  \bar\z^T(\M \bar\x - \bar\y)
	$$	
	for all $\x \in \R^d$, which establishes \eqref{eq: x stationary}.

	{
	  Let us look at the $\y$ update. Again let $g(\y) = \frac{\beta}{2} \y^T \y- \y^T (\beta \M \x^t + \z^t)$, then the optimality condition of this problem is \eqref{eq:y:opt}.
	Writing this out explicitly, and noting the fact that for any two vectors $\xi_1\in\partial \rho(\y)$ and $\xi_2\in\partial\rho(\y^{k+1})$ it must be true that $(\xi_1-\xi_2)^T (\y-\y^{k+1})\le \rho(\y)-\rho(\y^{k+1})$, we have that every iterate $\y^{k+1}$ satisfies
	$$
		\gamma (\rho(\y) - \rho(\y^{k+1})) + (\y - \y^{k+1})^T(\z^{k+1} + \beta\ \M(\x^{k+1} - \x^k) ) \ge 0
	$$
	for all $\y \in \R^p$.
	Taking the limit as $k \ra \infty$ shows that
	$$
		\gamma \rho(\y) + \bar \z ^T ( \M \bar \x - \y) \ge \gamma \rho(\bar \y) + \bar \z^T (\M \bar \x - \bar \y)
	$$
	as required.}
	This completes the proof. \qed
\end{proof}

\section{Extension to Sparse Principal Component Analysis and Penalized Eigenproblems}
\label{sec: SPCA ADMM}

{
{
The ADMM algorithm described by \eqref{eq: y step}-\eqref{eq: z step} may be applied to estimate the solution of \emph{any} penalized eigenproblem
of the form 
\begin{equation} \label{eq: generic prob}
	\max_{\x \in \R^p} \bra{ -\shalf \x^T \B \x + \gamma \| \D\x \|_1 : \W \x = \bs 0, \; \x^T \x \le 1}.
\end{equation}}
That is, we make no assumption that $\N \in \R^{p\times d}$
forms a basis for the null space of the within-class covariance matrix of a labeled data set, but rather assume that the columns
of $\N$ form an orthonormal basis
corresponding to some subspace constraint. 

The problem of identifying sparse solutions to eigenproblems has received significant attention, primarily
in relation to sparse principal component analysis.
In \emph{principal component analysis} (PCA) \cite[Section 14.5]{elements}, one seeks a dimensionality reduction maximizing variance in the lower dimensional space.
Specifically, the first $k$ principal components are the $k$ orthogonal directions $\w_1, \w_2, \dots, \w_k$
 maximizing $\w^T \bs{\tilde\Sigma} \w$, where $\bs{\tilde\Sigma} \in \S^p_+$
is an approximation of the population covariance matrix (typically the sample covariance matrix);
here $\S^p_+$ denotes the cone of $p\times p$ positive semidefinite matrices.
Thus, principal component analysis reduces to identifying the $k$ leading eigenvectors of the approximation of the
covariance matrix given by $\bs{\tilde \Sigma}$.
It is known that the sample covariance is a consistent estimator of the population covariance, i.e., the sample covariance matrix converges to the
true population covariance matrix with probability $1$ as the sample size $n$ tends to infinity
for fixed number of features $p$.
However, when $p$ is larger than $n$, as it is in the high dimension, low sample size setting,  the sample covariance
matrix may be a poor estimate of the population covariance;
see \cite{johnstone2004sparse,baik2006eigenvalues,paul2007asymptotics}.
One approach to addressing this issue, and to improve interpretability of the obtained loading vectors,
is to require
principal component vectors be sparse.
Many different methods for this task have been proposed, typically involving $\ell_0$ or $\ell_1$-regularization,
convex relaxation, thresholding, or some combination of all three;
see
\cite{jolliffe2003modified,zou2006sparse,d2007direct,d2008optimal,witten2009penalized,journee2010generalized,luss2011convex,yuan2011truncated,d2012approximation,bach2012structured,richtarik2012alternating,zhang2012sparse,papailiopoulos2013sparse} and the references within.

Our heuristic for LDA can be applied, with only minor modification, to a particular form of the
sparse PCA problem.
The leading principal component of a given data set can be identified by solving the optimization problem
\begin{equation} \label{eq: PCA}
	\w_1 = \argmax_{\w \in \R^p} \bra{ \w^T \bs{\hat\Sigma} \w : \w^T \w \le 1 },
\end{equation}
where $\bs{\hat\Sigma} \in \S^p_+$ is the sample covariance matrix of the given data (after centering).
A frequently used approach to simultaneously perform feature selection and principal component analysis
is to restrict the obtained principal component to be $s$-sparse, with respect to
the orthonormal basis $\D \in \O^p$,
for some integer $s$:
$$
	\w_1 = \argmax_{\w\in \R^p} \bra{ \w^T \bs{\hat\Sigma} \w: \w^T \w \le 1, \; \|\D \w\|_0 \le s };
$$	
here $\|\bv \|_0$ denotes the number of nonzero entries of the vector $\bv$.
Moving the cardinality constraint to the objective as a penalty and relaxing the $\ell_0$-norm with the $\ell_1$-norm yields the relaxation
\begin{equation} \label{eq: SPCA}
	\w_1 = \argmax_{\w \in \R^p} \bra{ \w^T \bs{\hat\Sigma} \w  -  \gamma \|\D\w\|_1 : \w^T \w \le 1 }.
\end{equation}	
Clearly \eqref{eq: SPCA} is a special case of \eqref{eq: PZFD} with $\B = \bs{\hat\Sigma}$ and $\W = \bs 0$  (or, equivalently, $\bs \N = \bs I$),
and $\bs\sigma$ {
equal to the all-ones vector of appropriate dimension}.
The relaxation \eqref{eq: SPCA} is very similar to the unified framework for sparse and functional principal component analysis of
Allen \cite{allen2013sparse}, although we employ a different heuristic approach than the proximal operator-based alternating direction method considered in \cite{allen2013sparse}.
On the other hand, alternating direction methods and the method of multipliers have been used repeatedly in the literature to solve relaxations of the sparse PCA problem,
most notably in \cite{ma2013alternating},
 partially motivated by the ease of evaluating the proximal operator of the $\ell_1$-norm.

We also note that our ADMM framework is immediately applicable to instances of \eqref{eq: generic prob}
for \emph{any} convex regularization function $\rho:\R^p \ra \R$.
For example, we may replace the $\ell_1$-penalty $\sum \sigma_i|y_i|$ used above with the group lasso $\ell_1/\ell_2$ penalty \cite{yuan2006model},
to induce discriminant vectors possessing
desired group sparsity structure.
However, we should remind the reader that the update step \eqref{eq: y step} requires evaluation of the proximal operator of
$\rho + \delta_{B_p}$, which does not admit an analytic formula for general $\rho$.

\section{Numerical Results}
\label{sec: expts}
\newcommand{\I}{\bs I}

We performed a series of numerical experiments to compare  our proposed algorithm (SZVD) with two recently proposed
heuristics for penalized discriminant analysis, namely the PLDA \cite{witten2011penalized} and SDA \cite{clemmensen2011sparse} methods
discussed in Section~\ref{sec: PLDA} implemented as the {\tt R} packages {\tt penalizedLDA}~\cite{PLDAR} and {\tt sparseLDA}~\cite{SDAR} respectively.

In each experiment, we learn sets of $k-1$ discriminant vectors from given training data using each heuristic,
and then test classification performance on a given test set. For each method, we apply validation to choose any regularization parameters.
{
In particular, we choose the regularization parameters to minimize classification error if the discriminant vectors have at most a desired number of nonzero features; if
the obtained discriminant vectors contain more nonzero entries than this tolerance, we use this number of nonzero entries as their validation score.}
We then use the regularization parameter $\gamma$ corresponding to the minimum validation score,
either classification error or number of nonzero features, to obtain our discriminant vectors.
For each data set and $i=1,\dots, k-1$, this validation, applied to our ADMM heuristic (SZVD), selects the regularization parameter $\gamma_i$ in \eqref{eq: PZFD xy} 
from a set of $m$ evenly spaced values in the interval $[0, \tilde \gamma_i]$,
where 
$$
	\tilde \gamma_i := \frac{(\w_0)_{i}^T \B (\w_0)_{i}}{\rho((\w_0)_i)}	
$$	
and $(\w_0)_{i}$ is the $i$th unpenalized zero-variance discriminant vector; this choice of $ \tilde\gamma_i$ is made to ensure that the problem \eqref{eq: PZFD xy}
has a nontrivial optimal solution by guaranteeing that at least one nontrivial solution with nonpositive objective value exists for each potential choice of $\gamma_i$.
Similarly, in PLDA we perform validation on the tuning parameter $\lambda$ controlling sparsity of the discriminant vectors in~\eqref{eq: WT LDA}.
Finally,  SDA employs two tuning parameters, $\lambda_1$, which controls the ridge regression penalty, and $\lambda_2$ ({\tt loads} in the R package),
which controls the number of nonzero features; in each experiment we fix $\lambda_1$ and perform validation to choose $\lambda_2$.
The choice of the weighting between classification performance and sparsity was chosen ad hoc to favor regularization parameters
which yield moderately sparse solutions with very good
misclassification rate; in experiments we observed that almost all classifiers found by SZVD had similar classification performance, regardless of regularization parameter, and we selected the parameter which gives the sparsest nontrivial solution via the validation step.

In all experiments, we use the dictionary matrix $\D=\I$, penalty weights $\bs\sigma = \diag(\W)$, regularization parameter $\beta = 2$,  and stopping tolerances
$tol_{abs} = tol_{rel} = 10^{-4}$ in  SZVD.
We initialized SZVD with  the unpenalized zero-variance discriminant vectors $\x_1^0 = \w_1^0, \dots, \x_{k-1}^0 = \w_{k-1}^0$
given by \eqref{eq: ZVD eig} and set $\y_i^0 = \D\N \x_i^0$ for all $i=1,2,\dots, k-1$;
consequently, we use the initial dual solution $\z_i^0 = \bs0$.
Using this initialization, SZVD can be thought of as a refinement of the ZVD solution to induce sparsity.
All features of discriminant vectors $\{\w_1, \w_2, \dots, \w_{k-1}\}$
found using SZVD with magnitude less than $0.025$ were rounded to 0,
{
and all solutions containing trivial discriminant vectors were discarded.}
All experiments were performed in {\tt R}; {\tt R} and {\tt Matlab} code for SZVD and 
for generating the synthetic data sets can be found at \cite{SZVD}.

\subsection{Simulated Data}
\label{sec: sims}

{
Three sets of simulations were considered.}
For each $k \in \{2,4\}$, we generate $25k,$ $25k$, and $250k$ training, validation, and test observations in $\R^{500 }$, respectively, as follows.
{
In the first two,} we sample 25, 25, and 250 Gaussian training, validation, and test observations from class $C_i$ from the distribution $N(\bs\mu_i, \bs\Sigma)$
for each $i \in \{1,2,\dots, k\}$,
where the mean vector $\bs\mu_i$ is defined by
$$
	[\bs\mu_i]_j = \branchdef{ 0.7, & \mbox{if } 100(i-1) + 1 \le j \le 100 i \\ 0, &\mbox{otherwise} }
$$
and the covariance matrix $\bs\Sigma$ is chosen in one of two ways:
\begin{itemize}
	\item
		In the first set of experiments, all features are correlated with
		$$
			[\bs\Sigma_r]_{ab}= \branchdef{ 1, &\mbox{if } a = b \\ r, &\mbox{otherwise.} }
		$$
		The experiment was repeated for each choice of $r\in \{0, 0.1, 0.5, 0.9\}$.
	\item
		In the second set of experiments, $\bs\Sigma$ is a block diagonal matrix with $100 \times 100$ diagonal blocks.
		For each $(a,b)$ pair with $a$ and $b$ belonging to the same block, we have
		$$
			[\bs\Sigma_\alpha]_{ab} = \alpha^{\abs{a - b}}.
		$$
		We let $[\bs\Sigma_\alpha]_{a b} = 0$ for all remaining $(a, b)$ pairs.
		The experiment was repeated for all choices of $\alpha \in \{0.1, 0.5, 0.9\}$.
\end{itemize}			

\noindent {
In the third, we sample $25$, $25$, and $250$ training, validation, and test observations from class $C_i$
from the multivariate t distribution, where we use $\nu = 1$ degrees of freedom, and the
mean vector $\bs\mu$ and covariance matrix $\bs\Sigma$ used in the underlying multivariate normal distribution
function are chosen as in the first set of experiments above.
To generate our random samples, we use the {\tt R} package {\tt mvtnorm} and the commands {\tt rmvnorm} (for the Gaussians)
and {\tt rmvt} (for the t distribution).}

For  each $(k, r)$ and $(k, \alpha)$ pair, we applied unpenalized zero-variance discriminant analysis (ZVD), our ADMM heuristic for penalized zero-variance  discriminant
analysis (SZVD), Witten and Tibshirani\rq{}s penalized linear discriminant analysis with $\ell_1$-norm and fused lasso penalties (PL1 and PFL),
and the SDA algorithm of Clemmensen et al.~(SDA) to obtain sets of $k-1$ discriminant vectors from the sampled training set.
These discriminant vectors were then used to perform dimensionality reduction of the test data, and each observation in the test set was assigned to the class of the nearest projected training class centroid;
an identical process was applied to the validation data to train any regularization parameters.
Both versions of PLDA chose the tuning parameter $\lambda$ using $20$ equally spaced values on the interval $[0, 0.15]$ by validation,
while SDA used the ridge regression parameter $\lambda_1 = 10$ and chose the sparsity tuning parameter {\tt loads} from
the set $$\{-500, -400,-300, -250, -200,-150, -120, -100, -80, -70, -60, -50\}$$ by validation.
These ranges of $\lambda$ and  {\tt loads} were chosen ad hoc to provided sets of potential discriminant vectors with a large range of classification
performance and numbers of nonzero features for the sake of training the regularization parameters.
The inner optimization of the SDA algorithm was stopped after a maximum of $5$ iterations  due to the high cost of each iteration.
{
The desired sparsity level used in the validation process was at most $35\%$ nonzero entries for each trial.}
This process was repeated $20$ times for each $(k,r)$ and $(k, \alpha)$ pair.
{All experiments were performed in {\tt R} v.3.2.0 using one node of the University of Alabama's cluster RC2,
containing two Intel Octa Core E5-2650 processors and 64GB of RAM.}
Tables~\ref{tab: Clem},~\ref{tab: WT},~and~\ref{tab: student t} report the average and standard deviation over all $20$ trials for each set of simulated data
of the number of misclassification errors, the number of nonzero features of discriminant vectors, and 
the time (in seconds) required to obtain each set of discriminant vectors per validation step
for each of the five methods;
the row headings \emph{Err}, \emph{Feat}, and \emph{Time} correspond to the number of misclassification
errors, the number of nonzero features, and the computation time, respectively.

\setlength{\tabcolsep}{4pt}


\begin{table}[t!] 
	\centering
	\resizebox{\textwidth}{!}{
		\begin{tabular}{llllllll}  \toprule \multicolumn{2}{l}{\bf Simulation 1}
									& ZVD 			& SZVD 			& PL1 			& PFL 		 	& SDA \\ \midrule[\heavyrulewidth]
				$k=2$	& \em Err		& 0.100 (0.447)	& 2.800 (4.697)	&  1.650  (2.323 )	& 0.000 (0.000)	& 0.250 (0.444) \\
				$r = 0$	& \em Feat 	& 490.350 (2.758)	& 136.850 (38.324)	& 121.750 (27.457)	& 186.300 (12.123)	& 149.350 (0.671) \\
						& \em Time 	& 0.494 (0.169)	& 0.468 (0.128)	&  0.017 (0.005)	& 0.024 (0.007)	& 1.212 (0.375) \\ \midrule
				$k=2$	& \em Err		& 0.100 (0.447) 	& 5.100 (6.912)	& 9.100 (12.965)	& 26.350 (26.990)	& 2.300 (8.430) \\
				$r=0.1 $	& \em Feat	& 490.950 (2.645)  	& 92.200 (27.156)	& 130.200 (25.622)	& 160.000 (39.901)	& 145.850 (15.749) \\
						& \em Time	& 0.504 (0.177)	& 0.488 (0.200)	& 0.014 (0.004)	& 0.022 (0.006)	& 1.118 (0.276)  \\ \midrule
				$k=2$	& \em Err		& 0.000 (0.000)	& 0.100 (0.447)	& 76.050 (59.797)	& 100.250 (63.519)	& 5.000 (11.671) \\
				$r=0.5$	& \em Feat	& 489.300 (3.147)	& 111.300 (24.312)	& 146.600 (30.474)	& 143.350 (45.968)	& 149.650 (0.587) \\
						& \em Time	& 0.512 (0.190)	& 0.473 (0.205) 	& 0.017 (0.005)	& 0.029 (0.011)	& 1.236 (0.389) \\ \midrule
				$k=2$	&  \em Err	&  0.000 (0.000)	& 0.000 (0.000)	& 140.600 (53.852)	& 156.550 (46.237)	& 29.950 (36.520)\\
				$r=0.9$	& \em Feat	& 483.700 (3.585)	& 138.650 (7.286)	& 137.050 (42.908)	& 119.100 (41.573)	& 149.200 (1.005) \\
						& \em Time	& 0.478 (0.156) 	& 0.423 (0.146)	& 0.015 (0.003)	& 0.027 (0.008)	& 1.129 (0.294) \\\midrule[\heavyrulewidth]
				$k=4$	& \em Err		& 1.000 (1.026)	& 26.800 (32.413)	& 2.550 (2.305) 	& 0.000 (0.000)	& 10.950 (16.340) \\
				$r=0	$	& \em Feat	& 1475.350 (5.851)	& 315.150 (39.833)	& 442.900 (131.772) & 385.350 (19.258)	& 568.800 (424.873) \\
						& \em Time	& 0.888 (0.314)	& 2.217 (0.697)	& 0.096 (0.030)	& 0.153 (0.043)	& 9.806 (4.014) \\ \midrule
				$k=4	$& \em Err	& 0.600 (0.883)	& 14.300 (7.168)	& 31.700 (30.418)	& 31.400 (29.564)	& 34.950 (37.870) \\
				$r=0.1$	& \em Feat	& 1472.400 (5.051)	& 350.700 (84.639)	& 1176.500 (84.754) & 1106.150 (315.328)	& 604.450 (398.613) \\
						&\em  Time		& 0.651 (0.166)	& 1.303 (0.075) & 0.064 (0.002)	& 0.098 (0.003)	& 5.942 (1.409)\\ \midrule
				$k=4$	&\em Err		& 0.000 (0.000)	& 1.150 (1.814)	& 283.300 (99.537)	& 279.900 (100.690) &31.500 (30.645) \\
				$r=0.5$	& \em Feat	& 1474.200 (4.641)	& 348.600 (57.135)	& 952.050 (471.463) & 940.150 (455.762)	& 509.050 (381.087) \\
						& \em Time	& 0.825 (0.263)	& 2.013 (0.653)	& 0.098 (0.032)	& 0.175 (0.067)	& 8.986 (2.921) \\ \midrule
				$k=4$	& \em Err		& 0.000 (0.000) 	& 0.000 (0.000)	& 402.300 (85.345)	& 399.600 (87.247)	& 1.150 (3.233) \\ 
				$r =0.9$	&\em  Feat	& 1472.200 (5.616)	& 449.050 (118.077)&754.400 (466.590) &771.750 (485.399)	& 667.000 (419.571) \\
						& \em Time	& 0.812 (0.276)	& 2.131 (0.768)	& 0.158 (0.069)	& 0.222 (0.088)	& 11.093 (4.220) \\ \bottomrule
					
		\end{tabular}	
	}	
	\caption{Comparison of performance  for synthetic data in $\R^{500}$ drawn from classes $C_1, \dots, C_k \sim N(\bs\mu_i, \bs{\Sigma_r})$
	where $\bs{\Sigma_r}$ is a $500\times 500$ matrix with diagonal equal to $1$ and all other entries equal to $r$. All values reported in the
	 format 	\qu{mean (standard deviation)}. In all trials, $N_{train} = N_{val} = 25k$, $N_{test} = 250 k$.}	
	\label{tab: Clem}
\end{table}

\begin{table}[t!] 
	\centering
	\resizebox{\textwidth}{!}{
		\begin{tabular}{lllllll} \toprule 
				\multicolumn{2}{l}{\bf Simulation 2}	& ZVD 			& SZVD 			& PL1 			& PFL 		 	& SDA \\ \midrule[\heavyrulewidth]
				$k=2$		& \em Err 			& 0.200 (0.616) 	& 8.200 (12.895)	& 3.250 (3.905)	& 0.000 (0.000) 	& 0.400 (0.681) \\
				$\alpha=0.1$	& \em Feat			& 490.300 (2.849) 	& 103.350 (24.684)	& 115.350 (31.938)	& 179.100 (26.887)	& 149.200 (0.894) \\
							& \em Time			& 0.297 (0.014)	& 0.291 (0.024)	& 0.010 (0.000)	& 0.014 (0.001)	& 0.731 (0.017) \\  \midrule
				$k=2$		& \em Err				& 10.100 (3.275)	& 28.500 (15.077)	& 9.450 (4.071)	& 6.550 (5.276)	& 9.000 (9.061) \\
				$\alpha=0.5$	& \em Feat			& 491.350 (2.796)	& 102.000 (25.452)	& 130.950 (23.442)	& 170.800 (26.877)	& 135.550 (33.325) \\
							& \em Time			& 0.302 (0.023)	& 0.291 (0.030)	& 0.010 (0.001)	& 0.015 (0.002)	& 0.737 (0.018) \\ \midrule
				$k=2$		& \em Err				& 212.500 (34.451)	& 223.700 (39.102)	& 89.500 (20.028)	& 91.300 (21.796) 	& 91.300 (17.499) \\
				$\alpha=0.9$	& \em Feat			& 490.550 (3.531)	& 106.350 (23.632)	& 137.250 (23.962)	& 143.650 (40.539)	& 99.100 (33.606) \\
							& \em Time			& 0.331 (0.074)	& 0.418 (0.112)	& 0.012  (0.004)	& 0.018 (0.004)	& 0.789 (0.119)  \\ \midrule[\heavyrulewidth]
				$k=4$		& \em Err				& 1.850 (1.424)	& 28.600 (7.423)	& 3.900 (3.076)	& 0.650 (1.137)	& 10.800 (14.468) \\
				$\alpha=0.1$	& \em Feat			&1474.150 (4.793)	& 306.300 (21.957)	& 423.350 (82.933)	& 383.300 (19.834)	& 542.250 (371.578) \\
							& \em Time			& 0.572 (0.046)	& 1.398 (0.137)	& 0.067 (0.006)	& 0.109 (0.022)	& 6.319 (1.385) \\ \midrule
				$k=4$		& \em Err				& 55.650 (12.287)	& 115.750 (22.856)	& 25.900 (12.397)	& 18.750 (8.410)	& 59.950 (43.584)\\
				$\alpha=0.5$	& \em Feat			& 1476.100 (6.290)	& 310.400 (29.170)	& 738.050 (307.163)& 721.300 (304.402)& 580.100 (487.021)\\
							& \em Time			& 0.596 (0.124)	& 1.451 (0.154)	& 0.067 (0.003)	& 0.106 (0.007)	& 7.052 (1.662) \\ \midrule
				$k=4$		& \em Errs			& 477.200 (28.957) & 492.900 (28.306)	& 300.100 (18.379)	& 299.650 (18.088)	& 380.400 (39.487)\\
				$\alpha=0.9$	& \em Feat			& 1474.200 (4.841) & 402.700 (26.872)	& 1056.300 (337.691)&1075.400 (314.011) & 256.650 (127.942) \\
							& \em Time			& 0.582 (0.022)	& 2.045 (0.103)	& 0.068 (0.005)	& 0.122 (0.066)	& 6.679 (0.468) \\ \bottomrule
		\end{tabular}	
	}	
	\caption{Comparison of performance for synthetic data in $\R^{500}$ drawn from classes $C_1, \dots, C_k \sim 
	N(\bs\mu_i, \bs{\Sigma_\alpha})$ where $\bs{\Sigma_\alpha}$ is a $500\times 500$ diagonal block matrix with 
	$100\times 100$ diagonal blocks with $(i,j)$ nonzero entry equal to $\alpha^{|i-j|}$. In all trials, $N_{train} = N_{val} = 25k$, $N_{test} = 250 k$.}
	\label{tab: WT}
\end{table}

{\color{blue}
\begin{table}[t!] 
	\centering
	\resizebox{\textwidth}{!}{
		\begin{tabular}{llllllll}  \toprule \multicolumn{2}{l}{\bf Simulation 3}
									& ZVD 			& SZVD 			& PL1 			& PFL 		 	& SDA \\ \midrule[\heavyrulewidth]
				$k=2$	& \em Err		& 106.500 (22.090) 	&147.500 (26.437) 	& 231.400 (32.173)	& 219.050 (48.424)	& 178.900 (60.225) \\
				$r = 0$	& \em Feat 	& 486.500 (3.348)	& 89.100 (19.040)	& 233.600 (79.962)	& 277.700 (86.270)	& 125.300 (37.480) \\
						& \em Time 	& 0.297 (0.020)	& 0.558 (0.096)	& 0.009 (0.005)	& 0.007 (0.001)	& 0.674 (0.027) \\ \midrule
				$k=2$	& \em Err		& 111.800 (28.298)	& 150.900 (31.692)	& 216.400 (41.140)	& 222.650 (36.152)	& 182.950 (61.908) \\
				$r=0.1 $	& \em Feat	& 486.600 (3.470)	& 99.300 (20.538)	& 227.150 (127.067) & 295.600 (54.472)	& 139.150 (18.466)\\
						& \em Time	& 0.289 (0.014)	& 0.539 (0.106)	& 0.007 (0.001)	& 0.007 (0.001)	& 0.676 (0.029)  \\ \midrule
				$k=2$	& \em Err		& 76.200 (29.890)	& 140.200 (123.300)& 224.850 (26.688)	& 223.000 (30.184)	& 181.750 (55.999)\\
				$r=0.5$	& \em Feat	& 488.550  (3.605)	& 102.950 (41.945)	& 235.000 (124.283)& 307.650 (99.222)	& 127.750 (32.598)\\
						& \em Time	& 0.287 (0.010)	& 0.508 (0.122)	& 0.008 (0.001)	& 0.008 (0.001)	& 0.647 (0.018)\\ \midrule
				$k=2$	&  \em Err	&  25.700 (8.832)	& 76.700 (140.006)	& 230.250 (22.178)	& 229.250 (24.244)	& 172.550 (86.696)\\
				$r=0.9$	& \em Feat	& 484.700 (3.466)	& 169.450 (58.531)	& 350.600 (134.481)& 364.300 (158.035) & 140.200 (12.903) \\
						& \em Time	& 0.284 (0.009)	& 0.364 (0.122)	& 0.007 (0.001)	& 0.011 (0.002)	& 0.630 (0.017) \\\midrule[\heavyrulewidth]
				$k=4$	& \em Err		& 206.700 (28.488)	& 254.650 (33.971)	& 689.550 (85.829)	& 677.750 (101.469) & 632.800 (128.331) \\
				$r=0	$	& \em Feat	& 1469.950 (4.883)	& 483.300 (128.665)& 610.900 (372.315) & 698.900 (259.122) & 580.600 (445.667)\\
						& \em Time	& 0.539 (0.014)	& 1.916 (0.428)	& 0.044 (0.002)	& 0.031 (0.003)	& 5.660 (1.169) \\ \midrule
				$k=4	$& \em Err	& 203.050 (40.382)	& 249.950 (36.767)	& 700.850 (53.635)	& 686.950 (61.275)	& 583.150 (108.370) \\
				$r=0.1$	& \em Feat	& 1468.900 (6.965)	& 467.400 (119.478)& 712.900 (388.527) & 873.750 (324.816)&539.850 (389.930) \\
						&\em  Time	& 0.538 (0.011)	& 1.842 (0.415)	&  0.043 (0.003)	& 0.032 (0.002)	& 5.495 (1.053) \\ \midrule
				$k=4$	&\em Err		& 131.450 (34.861)	& 204.900 (84.124)	& 696.600 (55.645)	& 691.400 (66.199)	& 608.450 (127.759)\\
				$r=0.5$	& \em Feat	& 1470.350 (6.467)	& 534.950 (187.540) & 671.200 (304.304) & 780.850 (329.883) & 570.950 (386.541) \\
						& \em Time	& 0.661 (0.212)	& 1.910 (0.625)	& 0.047 (0.008)	& 0.045 (0.007)	& 5.889 (1.160) \\ \midrule
				$k=4$	& \em Err		& 50.700 (13.929)	& 189.950 (177.184)	& 711.000 (43.623)	& 703.600 (48.309)	& 487.100 (228.919) \\ 
				$r =0.9$	&\em  Feat	& 1472.450 (5.246)	& 644.450 (142.944)	& 751.750 (252.882)	& 878.350 (336.339)	& 849.800 (446.208) \\
						& \em Time	& 0.691 (0.217)	& 1.868 (0.704)	& 0.045 (0.003)	& 0.058 (0.009)	& 6.444 (1.411)   \\ \bottomrule
					
		\end{tabular}	
		}
	\caption{Comparison of performance  for synthetic data in $\R^{500}$ drawn from multivariate t distribution as described above.
	All values reported in the  format 	\qu{mean (standard deviation)}. In all trials, $N_{train} = N_{val} = 25k$, $N_{test} = 250 k$.}	
	\label{tab: student t}
\end{table}
}

\subsection{Time-Series Data}
\label{sec: series}

\begin{table}[t!] 
	\centering
	\resizebox{\textwidth}{!}{
		\begin{tabular}{lllllll} \toprule \multicolumn{2}{l}{\bf  Time-Series}
										& ZVD 			& SZVD 			& PL1 			& PFL 		 	& SDA \\  \midrule[\heavyrulewidth]
			\bf \em OliveOil	& \em Err		& 1.900 (1.373)	& 2.900 (1.714)	& 3.700 (1.625)	& 4.100 (1.861)	& 1.950 (1.432)   \\
			$p=570$&\em Feat				& 1669.600 (6.723)	&  312.050 (75.076)	& 331.900 (156.725) & 343.750 (252.559) & 969.050 (358.902) \\
			$k=4$	& \em Time			& 1.052 (0.151) 	& 5.027 (0.799) 	&  0.020 (0.005)	& 0.046 (0.006)	& 5.788 (0.841) \\ \midrule
			\bf\em Coffee	& \em Err		& 0.000 (0.000) & 0.050  (0.224) &   6.600  (2.371) & 6.800  (2.462) & 0.250 (0.639)  \\
			$p=286$& \em Feat	  & 281.950 (2.212) &  44.250 (10.432) & 123.700 (56.610) & 129.250 (67.392) & 55.400 (25.046) \\
			$k=2$	& \em Time & 0.098 (0.015) & 0.269 (0.058) & 0.009 (0.001) & 0.013 (0.002) & 0.371 (0.010)  \\  \midrule
			\bf\em ECG	& \em Err		& 33.850 (18.253) & 45.100 (21.506) & 145.100 (26.989) & 152.350 (28.027) & 52.000 (24.921) \\
			$p=136$	& \em Feat	&  134.850 (0.671) &  32.200 (9.192) & 34.200 (11.190) & 38.300 (13.692) & 26.500 (11.821) \\
			$k=2$	& \em Time & 0.027 (0.008) & 0.068 (0.016) & 0.009 (0.001) & 0.014 (0.003) & 0.211 (0.019) \\ \bottomrule
		\end{tabular}	
	}					
	\caption{Comparison of performance for the {\em OliveOil, Coffee, and ECGFiveDays} data sets.}
	\label{tab: real data}
\end{table}

We performed similar experiments for three data sets drawn from the UCR time series data repository \cite{keogh2006ucr},
namely the {\em Coffee, OliveOil,} and {\em ECGFiveDays} data sets.
The {\em ECGFiveDays} data set consists of 136-dimensional electrocardiogram measurements of  a 67-year old male.
Each observation corresponds to a measurement of the electrical signal of a single heartbeat of the patient.
The data consists of two classes, 884 observations in total, corresponding to measurements taken on two dates, five days apart.
We randomly divided the data into training, validation, and testing sets containing $25, 100$, and $ 759$ observations, respectively.
We then applied each of our five methods to obtain discriminant vectors using each training and validation set pair and
 perform nearest centroid classification on the corresponding test set in the projected space.
The tuning parameter $\lambda$ in PLDA was selected from twenty equally spaced values in the interval $[0, 0.15]$, and
we set $\lambda_1 = 0.001$ and chose the tuning parameter {\tt loads} from the set
\begin{align*}
		\{-500, -400, -300, -250, -200, -150, -120,  -100, -80, -60, -50, -40, -30 , -20, -10\}  
\end{align*}
by validation when using SDA. As before, we stop the SDA inner optimization after 5 iterations.
{The desired sparsity level in the validation process was set to $25\%$ nonzero entries for each training, validation, testing split.}
We repeated this process for 20 (training, validation, testing)-splits of the data and recorded the results in Table~\ref{tab: real data}.

\newcommand{\oo}{{\em OliveOil\;}}
\newcommand{\coffee}{{\em Coffee\;}}
The \oo and \coffee data sets comprise $60$ and $56$ food spectrogram observations of different varieties of olive oil and coffee, respectively.
Here, mass spectroscopy is applied to generate signals (spectra) corresponding to the molecular composition of samples of either coffee or olive oil.
The goal is to distinguish between different varieties of olive oil and coffee from these spectral signals.
The \oo data set \cite{tapp2003ftir} consists of $570$-dimensional spectrograms corresponding
to samples of extra virgin olive oil from one of four countries (Greece, Italy, Portugal, or Spain);
$286$-dimensional spectrograms of either Arabica or Robusta variants of instant coffee
compose the \coffee data set \cite{briandet1996discrimination}.
As before, we divide the \oo data into training, validation, and testing sets containing
 $(30,10,20)$ observations, respectively.
 We then applied each of the five  approaches to learn a classification rule from the training and validation
 data and classify the given test data.
We used the same range of tuning parameter $\lambda$ as
 in the {\em ECGFiveDays} trials for each PLDA heuristic; we stopped the inner optimization step after 5 iterations, set $\lambda_1 = 0.1$, and used the same set
 of potential values of the tuning parameter {\tt loads}  as
 we did in {\em ECGFiveDays} trials for SDA. This process was repeated for $20$
 different data splits, and we then repeated the experiment for the \coffee data set
 using training, validation, test splits of size $(25,10,21)$.
 {For both {\em Coffee} and {\em OliveOil}, the desired sparsity level in the validation process was set to $35\%$ nonzero entries for each training, validation, testing split.}
 The results of these trials are summarized in Table~\ref{tab: real data}.
 
 \subsection{Commentary}
 \label{sec: comm}
 
 As can be seen from the experiments of the previous section, our proposed algorithm SZVD compares
 favorably to the current state of the art.
 When compared to the zero-variance discriminant, adding penalization in the form of a $\ell_1$-penalty results
 in a modest degradation in classification performance, as may be expected. However, this penalization significantly increased
 sparsity of the obtained discriminant vectors from that of the zero-variance discriminants.
 Moreover, SZVD significantly outperforms both forms of PLDA in terms of classification error,
 and exhibits similar performance to SDA while using significantly fewer computing resources.
In particular, our heuristic exhibits better classification performance, when compared to the three penalized LDA heuristics for all
but the uncorrelated synthetic data sets in Simulation~1, and two of the three real data sets.
Similarly, SZVD uses the fewest features in all of the four class trials in Simulation~1, and at least second fewest in all other trials.
Finally, SZVD uses less computing resources than all but the ZVD and PLDA heuristics, which perform poorly in terms of number of features
and classification, respectively.
It should also be noted that, although we did not verify that the conditions ensuring convergence of our ADMM heuristic given
by Theorem~\ref{thm: convergence} are satisfied (and there is no reason to expect them to be),
the ADMM heuristic converged in all trials after at most a few hundred iterations.
 
 There were two notable exceptions.
First, for uncorrelated data, i.e., the $r=0$ case in Simulation~1, both variants of PLDA outperformed SZVD in terms of classification error.
 This is not surprising, as the implicit assumption that the data are uncorrelated made when using the diagonal approximation holds
 for this special case.
 However, the performance of PLDA degrades significantly as $r$ is increased, while that of SZVD improves.
 Second, SZVD performs very poorly for highly correlated data in Simulation~2; roughly half of all
 test observations are misclassified in the $\alpha=0.9$ trials.
{It should be noted that SZVD performs only marginally worse than  unpenalized zero-variance discriminant analysis (ZVD) for $\alpha = 0.5$ and $\alpha = 0.9$;
there is a significantly sharper decrease in classification performance from ZVD to SZVD when the data are weakly correlated ($r=0$ and $r=1$ in Simulation 1 and $\alpha = 0.1$ in Simulation 2).
This suggests
that the classes may not be linearly separable in the null space of the within-class covariance matrix in this case.}
It should also be noted that none of the heuristics perform well for these particular synthetic data sets, with PLDA and SDA misclassifying at least one third of test observations
on average.
{Finally, none of the compared methods performed well, in terms of classification rate, when applied to the data sampled from multivariate t distributions. This is somewhat to be expected,
as Fisher's LDA is designed for symmetric distributions, such as the normal distribution, while the t distribution is asymmetric.
However, it should be noted that ZVD and SZVD exhibited significantly better classification performance than PLDA and SDA. }

The use of penalization, aside from encouraging sparsity of the discriminant vectors, also seems to increase interpretability of the discriminant vectors, although it should be noted that this claim is somewhat subjective. 
For example, the nonzero entries of the discriminant vector used to classify the {\em ECGFiveDays} data  align with
features where data in different classes appear to differ significantly (on average).
Specifically, both the zero-variance and SZVD discriminant vectors closely follow the trajectory of the
difference of the sample class-mean vectors $\bs\mu_1 - \bs\mu_2$.
However, most of the entries of the SZVD discriminant vector corresponding to small
magnitude entries of the zero-variance discriminant are set equal to zero;
the remaining nonzero entries of the SZVD discriminant vector correspond to features
where the two class mean vectors seem to differ the most significantly.
This is most apparent when comparing the discriminant vectors to the class mean vectors
of the data after centering and normalizing, although this phenomena is also weakly visible
when comparing class means for the original data set.
See Figure~\ref{fig: plots} for more details.
\begin{figure} [t]	
	\centering
	\subfloat[]{\includegraphics[width=0.5\textwidth]{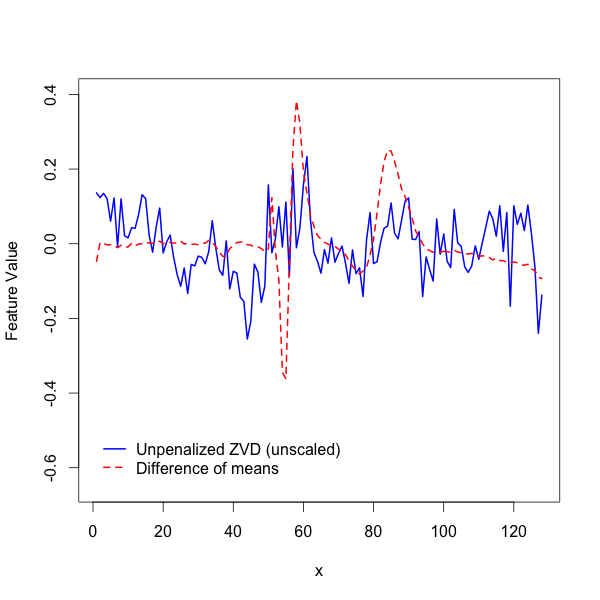} }  
	\subfloat[]{\includegraphics[width=0.5\textwidth]{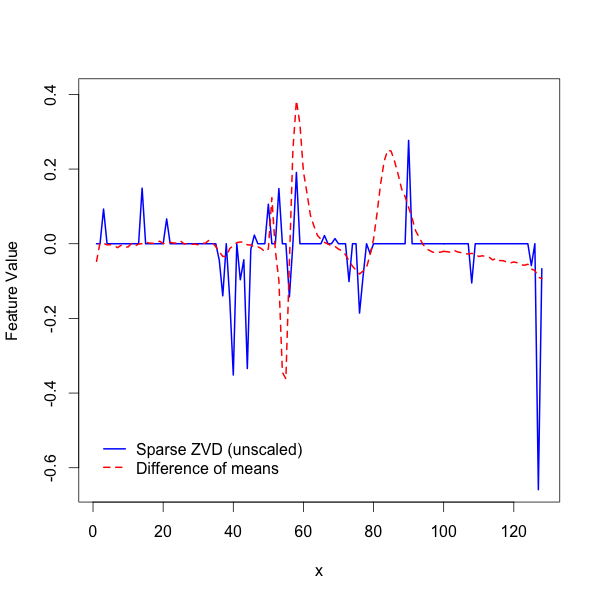}} 
	\newline
	\subfloat[]{\includegraphics[width=0.5\textwidth]{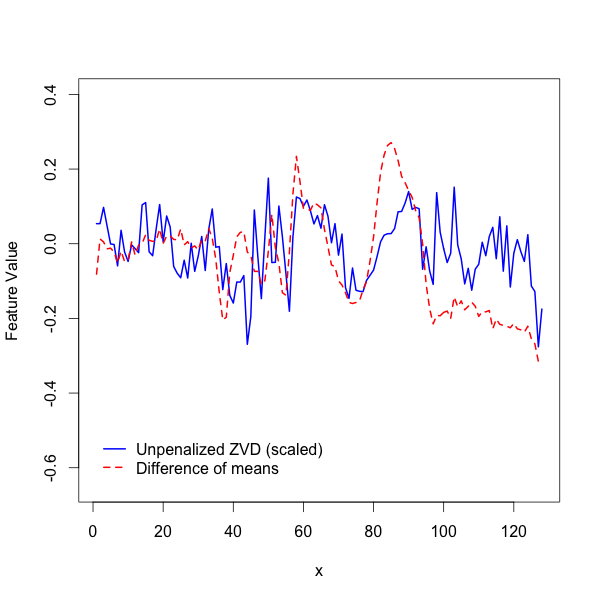} }  
	\subfloat[]{\includegraphics[width=0.5\textwidth]{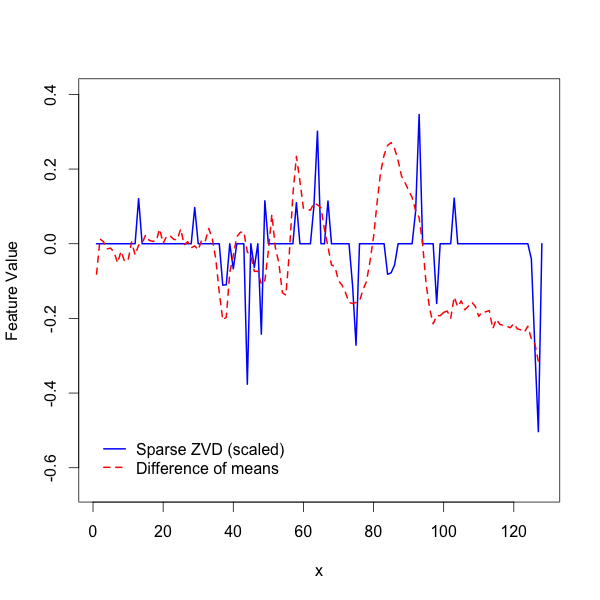}}
	
	\caption{Plots of the ZVD~(a) and SZVD~(b) discriminant vectors 
		with the difference of sample class-mean vectors (the red dashed-line) for the 
		raw {\em ECGFiveDays} data set;
		(c)~and~(d) plot the ZVD and SZVD discriminant vectors after normalizing and centering
		the {\em ECGFiveDays} data set.				
		The difference of sample class-mean vectors were rescaled in each plot to emphasize nonzero values of largest 
		magnitude aligning with large differences between the two classes.}		
	\label{fig: plots}
\end{figure}

\section{Conclusions}
\label{sec: conc}We have developed a novel heuristic for simultaneous feature selection and linear discriminant analysis based
on penalized zero-variance discriminant analysis and the alternating direction method of multipliers.
Our approach offers several advantages over the current state of the art.
Most notably, our algorithm, SZVD, employs inexpensive iterations, relative to those of the heuristics
employed in PLDA  \cite{witten2011penalized} and SDA \cite{clemmensen2011sparse}. One exception to this claim occurs when
the diagonal estimate of $\W$ is used in PLDA; in this case, PLDA is faster than our algorithm (as seen in the experiments of Section~\ref{sec: expts}).
However, our algorithm appears to provide significantly better classification performance than this particular version of PLDA, likely due to the use, by SZVD,
of information regarding correlation of features ignored by the diagonal estimate.
Moreover, our algorithm can be shown to converge
under some conditions on $\W$ and $\D$ (see Section~\ref{sec: convergence analysis}).
On the other hand, our approach is applicable, without modification, to the case when $k\ge 3$, unlike the linear programming approach of \cite{cai2011direct}.

This work opens several interesting areas of future research. For example, we establish convergence of the ADMM when applied to a nonconvex optimization problem by the novel use of the augmented Lagrangian as a measure of progress
in our convergence analysis. It would be extremely interesting to see if this approach and analysis can be extended to other classes of nonconvex optimization problems. On the other hand, it would be useful to strengthen our analysis to obtain estimates of the convergence rate of our algorithm and, if possible, establish convergence of SZVD in the absence of the assumptions
on $\W$ and $\D$ made in Section~\ref{sec: convergence analysis}. Finally, future research might include consistency analysis of penalized zero-variance analysis to identify when our heuristic is able to correctly classify data and identify important features.


\section{Acknowledgments}
{
The work presented in this paper was partially carried out while B.P.W Ames was a postdoctoral fellow at the Institute for Mathematics and its Applications during the IMA's annual program on Mathematics of Information (supported by National Science Foundation (NSF) award DMS-0931945), and while B.P.W. Ames was a Von Karman instructor at the California Institute of Technology supported by Joel Tropp under Office of Naval Research (ONR) award N00014-11-1002.
This research was also supported by University of Alabama Research Grant RG14678.
We  are grateful to Fadil Santosa, Krystal Taylor, Zhi-Quan Luo, Meisam Razaviyayn, and Line Clemmensen
 for their insights and helpful suggestions.
Finally, we are grateful for the contributions of two anonymous reviewers whose suggestions have greatly improved this 
manuscript.}


%




\bibliographystyle{abbrv}
\bibliography{1_bib}




\end{document}